\documentclass[11pt]{article}

\usepackage{fullpage}
\usepackage{amsthm}
\usepackage{amsmath}
\usepackage{natbib} 

\usepackage[T1]{fontenc}    
\usepackage{hyperref}       
\usepackage{url}            
\usepackage{booktabs}       
\usepackage{amsfonts}       
\usepackage{nicefrac}       
\usepackage{microtype}      
\usepackage{booktabs}

\usepackage{graphicx}
\usepackage[colorinlistoftodos]{todonotes}
\usepackage{hyperref}
\usepackage{multirow}

\usepackage{algpseudocode}
\usepackage[ruled,lined,boxed,commentsnumbered]{algorithm2e}
\usepackage{booktabs} 
\usepackage{float}

\newtheorem{theorem}{Theorem}[section]
\newtheorem{corollary}{Corollary}[theorem]
\newtheorem{lemma}[theorem]{Lemma}
\newtheorem{example}[theorem]{Example}
\newtheorem{definition}[theorem]{Definition}

\def\rc{\color{black}}
\def\bc{\color{black}}
\def\pc{\color{black}}

\def\x{{\bf x}}

\title{GESF: A Universal Discriminative Mapping Mechanism for Graph Representation Learning}

\author{
  Shupeng Gui\\
  University of Rochester\\
  \texttt{sgui2@ur.rochester.edu}
  \and
  Xiangliang Zhang\\
  KAUST, Saudi Arabia\\
  \texttt{Xiangliang.Zhang@kaust.edu.sa}
  \and
  Shuang Qiu\\
  University of Michigan\\
  \texttt{qiush@umich.edu}\\
  \and
  Mingrui Wu\\
  Alibaba Group\\
  \texttt{mingrui.wu@alibaba-inc.com}\\
  \and
  Jieping Ye\\
  University of Michigan\\
  \texttt{jieping@gmail.com}
  \and
  Ji Liu\\
  Tencent AI lab \\
  University of Rochester\\
  \texttt{ji.liu.uwisc@gmail.com}
}

\begin{document}

\maketitle

\begin{abstract}
{\bc
Graph embedding is a central problem in social network analysis and many other applications, aiming to learn the vector representation for each node. While most existing approaches need to specify the neighborhood and the dependence form to the neighborhood, which may significantly degrades the flexibility of representation, we propose a novel graph node embedding method (namely GESF) via the set function technique. Our method can 1) learn an arbitrary form of representation function from neighborhood, 2) automatically decide the significance of neighbors at different distances, and 3) be applied to heterogeneous graph embedding, which may contain multiple types of nodes.   Theoretical guarantee for the representation capability of our method has been proved for general homogeneous and heterogeneous graphs and evaluation results on benchmark data sets show that the proposed GESF outperforms the state-of-the-art approaches on producing node vectors for classification tasks.
}
\end{abstract}

\section{Introduction}\label{sec:introduction}
Graph node embedding is to learn a mapping that represents nodes as points in a low-dimensional vector space $\mathbb{R}^d$, where the geometric relationship reflects the structure of the original graph. Nodes that are ``close'' in the graph are embedded to have similar vector representations \citep{survey-CAI}. The learned node vectors benefit a number of graph analysis tasks, such as {\rc node classification \citep{bhagat2011node}, link prediction \citep{liben2007link}, community detection \citep{fortunato2010community}, and many others \citep{survey-jure}.}
In order to preserve the node geometric relations in {an} embedded space, the similarity/proximity/distance of a node to its \emph{neighborhood} is generally taken as input to different graph embedding approaches. For example, matrix-factorization approaches work on pre-defined pairwise similarity measures (e.g., different order of adjacency matrix). {\rc Deepwalk \citep{perozzi2014deepwalk}, node2vec \citep{grover2016node2vec} and other recent approaches \citep{dong2017metapath2vec}} consider flexible, stochastic measure of node similarity by the node co-occurrence on short random walks over the graph \citep{survey-Goyal}. Neighborhood autoencoder methods compress the information about a node's local neighborhood that is described as a neighborhood vector containing the node's pairwise similarity to all other nodes in the graph \citep{SDNE, DNGR}. Neural network based approaches such as  graph convolutional networks (GCN) and GraphSAGE apply convolution like functions on its surrounding neighborhood for aggregating neighborhood information \citep{kipf2016semi, GraphSAGE}.   

{\rc Although effective, all existing approaches need to specify the neighborhood and the dependence form to the neighborhood, which significantly degrades their flexibility for general graph representation learning. }
{\rc
In this work, we propose a novel graph node embedding method, namely Graph Embedding {\bc via} Set Function (GESF), that can 
\begin{itemize}
\item learn node representation via a \emph{universal} graph embedding function $f$, without pre-defining pairwise similarity, specifying random walk parameters, or choosing aggregation functions among element-wise mean, a max-pooling neural network, or LSTMs;
\item capture the arbitrary relationship between neighbors at different distance to the target node, and automatically decide the significance;
\item be generally applied to any graphs, from simple homogenous graphs to  heterogeneous graphs with complicated types of nodes.
\end{itemize}
}

{\pc
The core difficulty of graph node embedding is to characterize an arbitrary relationship to the neighborhood. 
From a local view of point, switching any two neighbors of a node in the same category would not affect the representation of this node. 
Based on this key observation, we propose to learn the embedding vector of a node via a \emph{partial permutation invariant set function} applied on its neighbors' embedding vector. We provide a neat form to represent such set function and prove that it can characterize an arbitrary partial permutation set function.
}
Evaluation results on benchmark data sets show that the proposed GESF outperforms the state-of-the-art approaches on producing node vectors for classification tasks.  
 
\section{Related Work}\label{sec:relatedwork}
The main difference among various graph embedding methods lies in how they define the ``closeness'' between two nodes \citep{survey-CAI}.  First-order proximity, second-order proximity or even high-order proximity have been widely studied for capturing the structural relationship between nodes \citep{tang2015line,yang2017fast}. In this section, we discuss the relevant graph embedding approaches in terms of how node closeness to neighboring nodes is measured, for highlighting our contribution on utilizing neighboring nodes in a most general manner. 
Comprehensive reviews of graph embedding can be found in \citep{survey-CAI, survey-jure, survey-Goyal,yang2017fast}.

\paragraph{Matrix Analysis on Graph Embeddding}
As early as 2011, a spectral clustering method \citep{tang2011leveraging} took the eigenvalue decomposition of a normalized Laplacian matrix of  a graph as an effective approach to obtain the embeddings of nodes. Other similar approaches work on different node similarity matrix by applying
various  similarity functions  to make a trade-off between modeling the ``first-order similarity'' and ``higher-order similarity'' \citep{GraRep, HOPE}. Node content information can also be easily fused in the pairwise similarity measure, e.g., in TADW \citep{yang2015network}, as well as node label information, which resulting in semi-supervised graph embedding methods, e.g., MMDW in \citep{tu2016max}.  

\paragraph{Random Walk on a Graph to Node Representation Learning}
Both deepwalk \citep{perozzi2014deepwalk} and node2vec \citep{grover2016node2vec} are outstanding graph embedding methods to solve the node representation learning problem. They convert the graph structures into a sequential context format with random walk \citep{lovasz1993random}. Thanks to the invention of \citep{mikolov2013distributed} for word representation learning of sentences, deepwalk inherited the learning framework for words representation learning in paragraphs to generate the representation of nodes in random walk context. And then node2vec evolved such the idea with additional hyper-parameter tuning for the trade-off between DFS and WFS to control the direction of random walk.
Planetoid \citep{yang2016revisiting} proposed a semi-supervised learning framework by guiding random walk with available node label information. 



\paragraph{Neighborhood Encoders to Graph Embedding}
{\bc There are also methods focusing on aggregating or encoding the neighbors' information to generate node embeddings. DNGR \citep{DNGR} and SDNE \citep{wang2016structural} introduce the autoencoder to construct the similarity function between the neighborhood vectors and the embedding of the target node. DNGR defines   neighborhood vectors based on random walks and SDNE introduces adjacency matrix and Laplacian eigenmaps to the definition of neighborhood vectors. Although the idea of autoencoder is a great improvement, these methods are painful when the scale of the graph is up to millions of nodes. Therefore, methods with neighborhood aggregation and convolutional encoders are involved to construct a local aggregation for node  embedding, such as GCN \citep{kipf2016semi, kipf2016variational, schlichtkrull2017modeling, van2017graph}, column networks \citep{pham2017column} and the GraphSAGE algorithm \citep{GraphSAGE}. The main idea of these methods is involving an iterative or recursive aggregation procedure e.g., convolutional kernels or pooling procedures to generate the embedding vectors for all nodes and such aggregation procedures are shared by all nodes in a graph.}

The above-mentioned methods work differently on using neighboring nodes for node representation learning. They require on pre-defining pairwise similarity measure between nodes, or specifying random walk parameters, or choosing aggregation functions. In practice, it takes  non-trivial effort to tune these parameters or try different measures, especially when graphs are complicated with nodes in multiple types, i.e., heterogeneous graphs. This work hence targets on making neighboring nodes play their roles in a most general manner such that their contributions are learned but not user-defined. The resultant embedding method has the flexibility to work on any types of homogeneous and heterogeneous graph.
The heterogeneity of graph nodes is handled by a heterogeneous random walk procedure in \citet{dong2017metapath2vec} and by deep neural networks in \citet{chang2015heterogeneous}.  
{\bc 
GESF has a natural advantage on avoiding any manual manipulation of random walking strategies or designs for the relationships between different types of nodes. To the invention of set functions in \citep{zaheer2017deep}, all existing valid mapping strategies from neighborhood to the target nodes can be represented by the set functions which are learnt by GESF automatically.} 
\section{A Universal Graph Embedding Model based on Set Function}
{\bc
In this section, 
we first introduce a universal mapping function to generate the embedding vectors for nodes in a graph via involving the neighborhood with various steps to the target nodes for the graph embedding learning and then we propose a permutation invariant set function as the universal mapping function. Sequentially, matrix function is introduced to process the knowledge of different orders of neighborhood. At last, we propose the overall learning model to solve the proper embedding vectors of nodes respect to a specific learning problem.
}

\subsection{A universal graph embedding model} 
We target on designing graph embedding models for the most general graph that may include $K$ different types of nodes. Formally, a graph $\mathcal{G}=\{\mathcal{V}, \mathcal{E}\}$, where the node set $\mathcal{V} = \cup_{k=1}^{K}\mathcal{V}_k$, i.e., $\mathcal{V}$ is composed of $K$ disjoint types of nodes.  One instance of such a graph is the academic publication network, which includes different types of nodes for papers, publication venues, author names, author affiliations, research domains etc.  

Given a graph $\mathcal{G}$, our goal is to learn the embedding vector  for each node in this graph. As we know, the position of a node in the embedded space is collaboratively determined by its neighboring nodes. Therefore, we propose a universal embedding model where the embedding vector $x^v$ of node $v \in \mathcal{V}_k$ can be represented by its neighbors' embedding vectors via a set function $f$
\[
x^v = f(X^v_1, X^v_2, \cdots, X^v_K)\quad \forall v \in \mathcal{V}_k,~\forall k \in [K]
\]
where $X^v_k$ is a matrix with column vectors corresponding to the  embedding   of node $v$'s neighbors in type $k$. Note that the neighbors can be step-1 (or immediate) neighbors, step-2 neighbors, or even higher degree neighbors. 
However, all neighboring nodes that are $\kappa$ steps reachable from a node play the same role when localizing this node in the embedded space. Therefore,   function $f(\cdot)$ should be a partially permutation invariant function. That is, if we swap any columns in each $X^v_k$, the function value remains the same. Unfortunately, the set function is not directly learnable due to the permutation property.

One straightforward idea to represent the partially permutation invariant function is to  define it in the following form
\begin{align}
f(X^v_1, \cdots, X^v_K) := \sum_{P_1\in \mathcal{P}_{|\mathcal{V}_1|}} \sum_{P_2\in \mathcal{P}_{|\mathcal{V}_2|}} \cdots \sum_{P_K\in \mathcal{P}_{|\mathcal{V}_K|}}  g(\bar{X}^v_1P_1, \cdots, \bar{X}^v_KP_K)
\label{eq:simple}
\end{align}
where $\mathcal{P}_{n}$ denotes the set of $n-$dimensional permutation matrices, and $\bar{X}^v_{1}, \cdots, \bar{X}^v_K$ denote the representation matrix consisting of the vectors in $X^v_1, \cdots, X^v_K$, respectively. $\bar{X}^v_1P_1$ is   to permute the columns in $\bar{X}^v_1$. It is easy to verify that the function defined in \eqref{eq:simple} is partially permutation invariant, but it is almost not learnable because it involves $\prod_{k=1}^N (|\mathcal{V}_k|\,!)$ ``sum'' items. 

Our solution of learning function $f$ is based on the following important theorem, which gives a neat manner to represent any partially permutation invariant function. The proof is in the Appendix. 

\begin{theorem} 
\label{theorem6}
Let $f$ be a continuous real-valued function defined on a compact set $X$ with the following form 
\[
f(\underbrace{x_{1,1}, x_{1,2}, \cdots, x_{1,N_1}}_{G_1}, \underbrace{x_{2,1}, x_{2,2}, \cdots, x_{2,N_2}}_{G_2}, \cdots, \underbrace{x_{K,1}, x_{K,2}, \cdots, x_{K,N_K}}_{G_K}).
\]
If function $f$ is partially permutation invariant, that is, any permutations of the values within the group $G_k$ for any $k$ does not change the function value, then there must exist functions $h(\cdot)$ and $\{g_k(\cdot)\}_{k=1}^K$ to approximate $f$ with arbitrary precision in the following form
\begin{align}
h\left(\sum_{n=1}^{N_1} g_1(x_{1,n}), \sum_{n=1}^{N_2} g_2(x_{2,n}), \cdots, \sum_{n=1}^{N_K} g_K(x_{K,n})\right).
\label{eq:URT}
\end{align}
\end{theorem}
Based on this theorem, we only need to parameterize $h(\cdot)$ and $\{g_k\}_{k=1}^K$ to learn the node embedding function. We next formulate the embedding model when considering different order of neighborhood. 
 
\paragraph{1-step neighbors}    From Theorem~\ref{theorem6},  \emph{any} mapping function of a node $v \in \mathcal{V}_k$ can be characterized by appropriately defined functions $\psi_1, \psi_2, \cdots, \psi_K$, and $\phi_k$:   
\[
x^v = \phi_{k} \left( \sum_{u\in \Omega^{v}_{1, 1}}  \psi_{1}\left( x^u\right), \sum_{u \in \Omega^{v}_{1, 2}}  \psi_{2}\left( x^u\right), \cdots, \sum_{u \in \Omega^{v}_{1, K}}  \psi_{K}\left( x^u\right) \right)\quad \forall v\in \mathcal{V}_k,~\forall k\in [K],
\]
where $\Omega^{v}_{n,k} := \Omega^v_{n} \cap \mathcal{V}_k$ denotes the step-$n$ neighbors of node $v$ in node type $k$.

\paragraph{Multi-step neighbors} 
High order proximity has been shown beneficial on generating high quality embedding vectors \citep{yang2017fast}. Extending the 1-step neighbor model, we can have the more general model where the representation of each node could depend on immediate (1-step) neighbors, 2-step neighbors, 3-step neighbors, and even infinite-step neighbors.
\begin{align}
x^v = \phi_{k} \left( \sum_{n=0}^{\infty} \alpha_n \sum_{u \in \Omega^{v}_{n, 1} }  \psi_{1}\left( x^u\right), ~\sum_{n=0}^{\infty} \alpha_n\sum_{u \in \Omega^{v}_{n,2}}  \psi_{2}\left( x^u\right), \cdots, \sum_{n=0}^{\infty} \alpha_n\sum_{u \in \Omega^{v}_{n, K}}  \psi_{K}\left( x^u\right) \right) 
& \label{eq:heter_m}
\\
\forall v\in \mathcal{V}_k~, \forall k \in [K]. & \nonumber
\end{align}
where $\alpha_1, \alpha_2, \cdots, \alpha_\infty$ are the weights for neighbors at different steps. {\rc Let $A\in \{0,1\}^{|\mathcal{V}|\times |\mathcal{V}|}$ be the adjacent matrix indicating all edges by $1$. If we define the polynomial matrix function $\mathfrak{P}(\cdot)$ on the adjacent matrix $A$ as $\mathfrak{P}(A) = \sum_{n=0}^{\infty} \alpha_n A^n$,
we can cast \eqref{eq:heter_m} into its matrix form
\begin{align}
x^v = \phi_{k} \left( \bar{\psi}_{1}(X_1)[\mathfrak{P}(A)]_{\mathcal{V}_1, v},~\bar{\psi}_{2}(X_2)[\mathfrak{P}(A)]_{\mathcal{V}_2, v}, ~\cdots,~ \bar{\psi}_{K}(X_K)[\mathfrak{P}(A)]_{\mathcal{V}_K, v} \right) &\label{eq:heter_mm} \\
\quad \forall v\in \mathcal{V}_k~, \forall k \in [K]. &\nonumber
\end{align}
where $X_k$ denotes the representation matrix for nodes in type $k$, 
$[\mathfrak{P}(A)]_{\mathcal{V}_k, v}$ denotes the submatrix of $\mathfrak{P}(A)$ indexed by column $v$ and rows in $\mathcal{V}_k$, and function $\bar{\phi}$ (with $\bar{\cdot}$ on the top of function $\phi()$) is defined as the function extension
\[
\bar{\phi}(X) := [\phi(x_1), \phi(x_2), \cdots, \phi(x_{N})].
\]
Note that the embedding vectors for different type of nodes may be with different dimensions.
Homogeneous graph is a special case of the heterogeneous graph with $K=1$. The above proposed model is thus naturally usable  on homogeneous graph. 

To avoid optimizing infinite number of coefficients, we propose to use a 1-dimensional NN function $\rho(\cdot):~\mathbb{R} \rightarrow \mathbb{R}$ to equivalently represent the function $\mathfrak{P}(\cdot)$ to reduce the number of parameters based on the following observations

\[
\mathfrak{P}(A)  =
U
\text{diag}\left(\sum_{n=0}^\infty \alpha_n \sigma_1^n, \cdots, \sum_{n=0}^\infty \alpha_n \sigma_N^n\right)
U^\top
 = 
 U
\text{diag}\left(\rho(\sigma_1), \cdots, \rho(\sigma_N)\right)
U^\top,
\]
where $A=U\text{diag}(\sigma_1, \cdots, \sigma_N)U^\top$ is the singular value decomposition. 
We parameterize $\rho: \mathbb{R} \mapsto \mathbb{R}$ using 1-dimensional NN, which allows us easily controlling the number of variables to optimize in $\rho$ by choosing the number of layers and the number of nodes in each layer.

}

\subsection{The overall model}

{\bc
For short, we denote the representation function for $x^v$ in \eqref{eq:heter_mm} by
\[
x^v = \mathcal{R}_{\rho, \phi^k, \{\psi^k\}_{k=1}^K}(v) \quad \forall v\in \mathcal{V}_k~, \forall k \in [K]. 
\]
To fulfill the requirement of a specific learning task, we propose the following learning model involving a supervised component
\begin{align}\label{eq:over-all}
  \min_{
    \substack{
      \{x^v\}_{v\in \mathcal{V}}, \rho, \\
      \{\phi^k\}_{k=1}^K, \{\psi^k\}_{k=1}^K}
      } \quad 
    \sum_{k=1}^K \frac{1}{\lambda_k|\mathcal{V}_k|}\sum_{u\in \mathcal{V}_k} \left\|x^u -  \mathcal{R}_{\rho, \phi^k, \{\psi^k\}_{k=1}^K}(u) \right\|^2 
    + \frac{1}{|\mathcal{V}_{\text{label}}|}\sum_{v\in \mathcal{V}_{\text{label}}}\ell(x^v, y^v)
  \end{align}
where $\mathcal{V}_{\text{label}} \subset \mathcal{V}$ denotes the set of labeled nodes, and $\lambda_k>0$ balances the representation error and prediction error. The first unsupervised learning component restricts the representation error between the target node and its neighbors with $L_2$ norm since it is allowed to have noise in a practical graph. And the supervised component is flexible to be replaced with any designed learning task on the nodes in a graph. For example, to a regression problem, a least square loss can be chosen to replace $\ell(x, y)$ and a cross entropy loss can be used to formulate a classification learning problem. To solve the problem in Eq. \eqref{eq:over-all}, we apply a stochastic gradient descent algorithm (SGD) to compute the effective solutions for the learning variables simultaneously.
}

\section{Experiments}

This section reports experimental results to validate the proposed method, comparing to   state-of-the-art algorithms on  benchmark datasets including both homogenous and heterogenous graphs.




 \subsection{Comparison on homogeneous graphs}
We consider the multi-class classification problem over the homogeneous graphs. Given a graph with partially labeled nodes, the goal is to learn the representation for each node for predicting the class for unlabelled nodes.


\paragraph{Datasets}
We evaluate the performance of GESF and methods for comparison on five datasets.   
{\small
\begin{itemize}
\item \textbf{Cora} \citep{mccallum2000automating} is a paper citation network. Each node is a paper. There are 2708 papers and 5429 citation links in total. Each paper is associated with one of 7 classes.
\item \textbf{CiteSeer} \citep{giles1998citeseer} is another paper citation network. CiteSeer contains 3,312 papers and 4,732 citations in total. All these papers have been classified into 6 classes.
\item \textbf{Pubmed} \citep{sen2008collective} is a larger and more complex citation networks compared to previous two datasets. There are 19,717 vertexes and 88,651 citation links. Papers are classified into 3 classes.
\item \textbf{Wikipedia} \citep{sen2008collective} contains 2,405 online webpages. The 17,981 links are undirected between pairs of them. All these pages are from 17 categories.
\item \textbf{Email-eu} \citep{leskovec2007graph} is an Email communication network which illustrates the email communication relationships between researchers in a large European research institution. There are 1,005 researchers and 25,571 links between them. Department affiliations (42 in total) of the researchers are considered as labels to predict. 
\end{itemize}
}

\paragraph{Baseline methods}
The compared baseline algorithms are listed below:
{\small
\begin{itemize}
\item \textbf{Deepwalk} \citep{perozzi2014deepwalk} is an unsupervised graph embedding method which relies on the random walk and word2vec method.  For each vertex, we  take 80 random walks with length 40, and set windwo size 10. Since deepwalk is {\bf unsupervised}, we apply a logistic regression on the generated embeddings for node classification. 
\item \textbf{Node2vec} \citep{grover2016node2vec} is an improved graph embedding method based on deepwalk. We set the window size as 10, the walk length as 80 and the number of walks for each node is set to 100. Similarly, the node2vec is {\bf unsupervise}d as well. We apply the same evaluation procedure on the embeddings of node2vec as what we did for deepwalk.
\item \textbf{MMDW} \citep{tu2016max} is a {\bf semi-supervised} learning framework of graph embedding which combines matrix decomposition and SVM classification. We tune the method multiple times and take 0.01 as the hyper-parameter $\eta$ in the method which is recommended by authors.
\item \textbf{Planetoid} \citep{yang2016revisiting} is a {\bf semi-supervised} learning framework.
We mute the node attributes involved in \textit{planetoid} since we only focus on the information abstraction from the graph structures.
\item \textbf{GCN (Graph Convolutional Networks)} \citep{kipf2016semi} applies the convolutional neural networks into the {\bf semi-supervised} embedding learning of graph. 
We eliminate the node attributes for fairness as well.
\end{itemize}
}

\paragraph{Experiment setup and results}
For fair comparison, the dimension of representation vectors is chosen to be the same for all algorithms (the dimension is $64$).
The hyper-parameters are fine-tuned for all of them. The details of GESF for multi-class case are as follows.
{\small
\begin{itemize}
  \item \textbf{Supervised Component:} \textit{Softmax} function is chosen to formulate our supervised component in Eq. \eqref{eq:over-all} which is defined as $\sigma:\mathbb{R}^K \mapsto \{z\in\mathbb{R}^K | \sum_{i=1}^K z_i = 1, z_i > 0\}$. For an arbitrary embedding $x\in\mathbb{R}^d$, we have the probability term as $\text{P}(y=j|x) = \frac{\exp(w_j^\top x + b_j)}{\sum_i^K \exp(w_i^\top x + b_i)}$ for such node to be predicted as class $j$, where $(w_j\in\mathbb{R}^{d}, b_j\in\mathbb{R})$ is a classifier for class $j$. Therefore, the whole supervised component in Eq. \eqref{eq:over-all} is $-\sum_{v\in\mathcal{V}_{\textrm{label}}}\sum_j^K y_j \log \text{P}(y=j | x_v) + \lambda_w\mathcal{R}(w)$, where $\mathcal{R}(w)$ is an $L_2$ regularization for $w$ and $\lambda_w$ is chosen to be $10^{-3}$.
  \item \textbf{Unsupervised embedding mapping Component:} We design a two-layer NN with hidden dimension 64 to form the mapping from embeddings of neighbors to the target node and we also form a two-layer 1-to-1 NN with a 3 dimensional hidden layer to construct the matrix function for the adjacency matrix. We pre-process the matrix $A$ with an eigenvalue decomposition where the procedure preserve the highest 1000 eigenvalues in default. The balance hyper-parameter $\lambda_1$ is set to be $10^{-3}$.
\end{itemize}
}

We take   experiments on each data set and compare the performance among all methods mentioned above. 
Since it is the multi-class classification scenario, we use \textit{Accuracy} as the evaluation criterion. The percentage of labeled samples is chosen from $10\%$ to $90\%$ and the remaining samples are used for evaluation, except for \textit{planetoid} \citep{yang2016revisiting}. Fixed   training   and testing dataset are used due to the optimization strategy of \textit{planetoid}, which is dependent upon the matching order of the vertexes in the graph and the order in the training and test set.
Therefore, we only provide the results of \textit{planetoid} of one time. All experiments are repeated for {three} times and we report the mean and standard deviation of their performance in the Tables \ref{tab:acc-cora}
.  We highlight the best performance for each dataset with \textbf{bold} font style and the second best results with a ``*''. We can observe that in most cases, our method outperforms other methods. 

\begin{table}[!htp]
  \centering
  {\scriptsize
  \caption{Accuracy (\%) of Multi-class Classification Experiments}
  \label{tab:acc-cora}
  \begin{tabular}{@{}ccccccccccc@{}}
  \toprule
  \multicolumn{1}{l}{}&training\%  & 10.00\%                                                      & 20.00\%                                                      & 30.00\%                                                               & 40.00\%                                                               & 50.00\%                                                               & 60.00\%                                                               & 70.00\%                                                               & 80.00\%                                                               & 90.00\%                                                               \\ \midrule
  \multicolumn{1}{c|}{\multirow{6}{*}{\rotatebox{90}{Cora}}}&deepwalk        & {\begin{tabular}[c]{@{}c@{}}75.47*\\ $\pm$1.01\end{tabular}} & {\begin{tabular}[c]{@{}c@{}}78.29*\\ $\pm$1.30\end{tabular}} & \begin{tabular}[c]{@{}c@{}}79.43\\ $\pm$1.29\end{tabular}          & \begin{tabular}[c]{@{}c@{}}79.96\\ $\pm$0.76\end{tabular}          & \begin{tabular}[c]{@{}c@{}}80.80\\ $\pm$0.40\end{tabular}          & \begin{tabular}[c]{@{}c@{}}81.33\\ $\pm$0.82\end{tabular}          & \begin{tabular}[c]{@{}c@{}}81.50\\ $\pm$0.69\end{tabular}          & \begin{tabular}[c]{@{}c@{}}80.89\\ $\pm$0.82\end{tabular}          & \begin{tabular}[c]{@{}c@{}}83.04\\ $\pm$2.47\end{tabular}          \\
  \multicolumn{1}{c|}{}&node2vec        & \textbf{\begin{tabular}[c]{@{}c@{}}75.52\\ $\pm$1.22\end{tabular}} & \begin{tabular}[c]{@{}c@{}}77.81\\ $\pm$1.51\end{tabular} & \begin{tabular}[c]{@{}c@{}}78.91\\ $\pm$0.94\end{tabular}          & \begin{tabular}[c]{@{}c@{}}79.72\\ $\pm$0.44\end{tabular}          & \begin{tabular}[c]{@{}c@{}}80.86\\ $\pm$0.67\end{tabular}          & \begin{tabular}[c]{@{}c@{}}80.70\\ $\pm$1.32\end{tabular}          & \begin{tabular}[c]{@{}c@{}}80.37\\ $\pm$1.22\end{tabular}          & \begin{tabular}[c]{@{}c@{}}80.59\\ $\pm$1.03\end{tabular}          & \begin{tabular}[c]{@{}c@{}}81.78\\ $\pm$1.41\end{tabular}          \\
  \multicolumn{1}{c|}{}&MMDW            & \begin{tabular}[c]{@{}c@{}}74.88\\ $\pm$0.23\end{tabular} & \textbf{\begin{tabular}[c]{@{}c@{}}79.18\\ $\pm$0.10\end{tabular}} & {\begin{tabular}[c]{@{}c@{}}81.20*\\ $\pm$0.11\end{tabular}}          & {\begin{tabular}[c]{@{}c@{}}82.19*\\ $\pm$0.25\end{tabular}}          & {\begin{tabular}[c]{@{}c@{}}83.10*\\ $\pm$0.41\end{tabular}}          & {\begin{tabular}[c]{@{}c@{}}84.62*\\ $\pm$0.09\end{tabular}}          & {\begin{tabular}[c]{@{}c@{}}85.54*\\ $\pm$0.44\end{tabular}}          & {\begin{tabular}[c]{@{}c@{}}85.27*\\ $\pm$0.22\end{tabular}}          & {\begin{tabular}[c]{@{}c@{}}87.82*\\ $\pm$0.37\end{tabular}}          \\
  \multicolumn{1}{c|}{}&planetoid & \begin{tabular}[c]{@{}c@{}}51.31$\pm$0\end{tabular}      & \begin{tabular}[c]{@{}c@{}}50.62$\pm$0\end{tabular}      & \begin{tabular}[c]{@{}c@{}}47.63$\pm$0\end{tabular}               & \begin{tabular}[c]{@{}c@{}}44.49$\pm$0\end{tabular}               & \begin{tabular}[c]{@{}c@{}}36.12$\pm$0\end{tabular}               & \begin{tabular}[c]{@{}c@{}}27.77$\pm$0\end{tabular}               & \begin{tabular}[c]{@{}c@{}}29.15$\pm$0\end{tabular}               & \begin{tabular}[c]{@{}c@{}}29.15$\pm$0\end{tabular}               & \begin{tabular}[c]{@{}c@{}}28.41$\pm$0\end{tabular}               \\
  \multicolumn{1}{c|}{}&GCN       & \begin{tabular}[c]{@{}c@{}}29.98\\ $\pm$0.26\end{tabular} & \begin{tabular}[c]{@{}c@{}}30.67\\ $\pm$0.53\end{tabular} & \begin{tabular}[c]{@{}c@{}}30.75\\ $\pm$0.51\end{tabular}          & \begin{tabular}[c]{@{}c@{}}29.66\\ $\pm$0.84\end{tabular}          & \begin{tabular}[c]{@{}c@{}}30.35\\ $\pm$0.59\end{tabular}          & \begin{tabular}[c]{@{}c@{}}30.84\\ $\pm$0.42\end{tabular}          & \begin{tabular}[c]{@{}c@{}}31.08\\ $\pm$1.36\end{tabular}          & \begin{tabular}[c]{@{}c@{}}31.49\\ $\pm$1.05\end{tabular}          & \begin{tabular}[c]{@{}c@{}}27.90\\ $\pm$1.86\end{tabular}          \\
  \multicolumn{1}{c|}{}&\textbf{GESF}   & \begin{tabular}[c]{@{}c@{}}69.58\\ $\pm$2.12\end{tabular} & \begin{tabular}[c]{@{}c@{}}78.08\\ $\pm$1.16\end{tabular} & \textbf{\begin{tabular}[c]{@{}c@{}}81.37\\ $\pm$2.01\end{tabular}} & \textbf{\begin{tabular}[c]{@{}c@{}}84.16\\ $\pm$0.41\end{tabular}} & \textbf{\begin{tabular}[c]{@{}c@{}}85.60\\ $\pm$1.49\end{tabular}} & \textbf{\begin{tabular}[c]{@{}c@{}}85.66\\ $\pm$1.16\end{tabular}} & \textbf{\begin{tabular}[c]{@{}c@{}}86.17\\ $\pm$0.38\end{tabular}} & \textbf{\begin{tabular}[c]{@{}c@{}}87.68\\ $\pm$0.46\end{tabular}} & \textbf{\begin{tabular}[c]{@{}c@{}}88.27\\ $\pm$0.77\end{tabular}} \\ \hline
  \multicolumn{1}{c|}{\multirow{6}{*}{\rotatebox{90}{CiteSeer}}}&deepwalk        & \begin{tabular}[c]{@{}c@{}}51.66\\ $\pm$0.73\end{tabular} & \begin{tabular}[c]{@{}c@{}}54.68\\ $\pm$1.28\end{tabular} & \begin{tabular}[c]{@{}c@{}}56.15\\ $\pm$0.64\end{tabular} & \begin{tabular}[c]{@{}c@{}}56.59\\ $\pm$0.37\end{tabular} & \begin{tabular}[c]{@{}c@{}}56.92\\ $\pm$1.04\end{tabular} & \begin{tabular}[c]{@{}c@{}}57.78\\ $\pm$0.88\end{tabular} & \begin{tabular}[c]{@{}c@{}}56.92\\ $\pm$0.70\end{tabular}          & \begin{tabular}[c]{@{}c@{}}56.71\\ $\pm$1.39\end{tabular}          & \begin{tabular}[c]{@{}c@{}}59.40\\ $\pm$0.49\end{tabular}          
    \\
  \multicolumn{1}{c|}{}&node2vec        & \begin{tabular}[c]{@{}c@{}}51.97*\\ $\pm$1.09\end{tabular} & \begin{tabular}[c]{@{}c@{}}55.17*\\ $\pm$1.22\end{tabular} & \begin{tabular}[c]{@{}c@{}}56.28\\ $\pm$0.41\end{tabular} & \begin{tabular}[c]{@{}c@{}}55.82\\ $\pm$1.03\end{tabular} & \begin{tabular}[c]{@{}c@{}}57.44\\ $\pm$0.54\end{tabular} & \begin{tabular}[c]{@{}c@{}}57.49\\ $\pm$1.07\end{tabular} & \begin{tabular}[c]{@{}c@{}}57.42\\ $\pm$1.38\end{tabular}          & \begin{tabular}[c]{@{}c@{}}56.89\\ $\pm$1.87\end{tabular}          & \begin{tabular}[c]{@{}c@{}}57.95\\ $\pm$1.07\end{tabular}          
    \\
  \multicolumn{1}{c|}{}&MMDW            & \textbf{\begin{tabular}[c]{@{}c@{}}55.36\\ $\pm$0.60\end{tabular}} & \textbf{\begin{tabular}[c]{@{}c@{}}60.98\\ $\pm$0.56\end{tabular}} & \textbf{\begin{tabular}[c]{@{}c@{}}62.00\\ $\pm$0.17\end{tabular}} & \textbf{\begin{tabular}[c]{@{}c@{}}63.89\\ $\pm$0.12\end{tabular}} & \textbf{\begin{tabular}[c]{@{}c@{}}66.59\\ $\pm$0.27\end{tabular}} & \begin{tabular}[c]{@{}c@{}}69.00*\\ $\pm$0.15\end{tabular} & \begin{tabular}[c]{@{}c@{}}69.72*\\ $\pm$0.82\end{tabular}          & \begin{tabular}[c]{@{}c@{}}70.40*\\ $\pm$0.93\end{tabular}          & \begin{tabular}[c]{@{}c@{}}70.64*\\ $\pm$0.31\end{tabular}          
    \\
  \multicolumn{1}{c|}{}&planetoid & \begin{tabular}[c]{@{}c@{}}41.53$\pm$0\end{tabular}      & \begin{tabular}[c]{@{}c@{}}41.62$\pm$0\end{tabular}      & \begin{tabular}[c]{@{}c@{}}40.19$\pm$0\end{tabular}      & \begin{tabular}[c]{@{}c@{}}37.88$\pm$0\end{tabular}      & \begin{tabular}[c]{@{}c@{}}32.91$\pm$0\end{tabular}      & \begin{tabular}[c]{@{}c@{}}29.06$\pm$0\end{tabular}      & \begin{tabular}[c]{@{}c@{}}20.82$\pm$0\end{tabular}               & \begin{tabular}[c]{@{}c@{}}21.11$\pm$0\end{tabular}               & \begin{tabular}[c]{@{}c@{}}21.39$\pm$0\end{tabular}               
    \\
  \multicolumn{1}{c|}{}&GCN       & \begin{tabular}[c]{@{}c@{}}21.00\\ $\pm$0.05\end{tabular} & \begin{tabular}[c]{@{}c@{}}19.07\\ $\pm$1.40\end{tabular} & \begin{tabular}[c]{@{}c@{}}20.93\\ $\pm$0.28\end{tabular} & \begin{tabular}[c]{@{}c@{}}20.91\\ $\pm$0.44\end{tabular} & \begin{tabular}[c]{@{}c@{}}19.52\\ $\pm$0.72\end{tabular} & \begin{tabular}[c]{@{}c@{}}20.20\\ $\pm$0.68\end{tabular} & \begin{tabular}[c]{@{}c@{}}20.54\\ $\pm$2.91\end{tabular}          & \begin{tabular}[c]{@{}c@{}}20.90\\ $\pm$1.82\end{tabular}          & \begin{tabular}[c]{@{}c@{}}18.78\\ $\pm$1.49\end{tabular}          
    \\
  \multicolumn{1}{c|}{}&\textbf{GESF}            & \begin{tabular}[c]{@{}c@{}}46.50\\ $\pm$2.61\end{tabular} & \begin{tabular}[c]{@{}c@{}}53.26\\ $\pm$1.62\end{tabular} & \begin{tabular}[c]{@{}c@{}}59.68*\\ $\pm$0.43\end{tabular} & \begin{tabular}[c]{@{}c@{}}62.47*\\ $\pm$1.32\end{tabular} & \begin{tabular}[c]{@{}c@{}}66.41*\\ $\pm$1.20\end{tabular} & \textbf{\begin{tabular}[c]{@{}c@{}}69.76\\ $\pm$0.49\end{tabular}} & \textbf{\begin{tabular}[c]{@{}c@{}}71.77\\ $\pm$2.13\end{tabular}} & \textbf{\begin{tabular}[c]{@{}c@{}}72.16\\ $\pm$1.55\end{tabular}} & \textbf{\begin{tabular}[c]{@{}c@{}}77.64\\ $\pm$1.69\end{tabular}} \\ \hline
  \multicolumn{1}{c|}{\multirow{6}{*}{\rotatebox{90}{Pubmed}}}&deepwalk        & \begin{tabular}[c]{@{}c@{}}76.98*\\ $\pm$0.33\end{tabular} & \begin{tabular}[c]{@{}c@{}}77.48\\ $\pm$0.20\end{tabular} & \begin{tabular}[c]{@{}c@{}}77.68\\ $\pm$0.09\end{tabular}          & \begin{tabular}[c]{@{}c@{}}77.72\\ $\pm$0.17\end{tabular}          & \begin{tabular}[c]{@{}c@{}}77.99\\ $\pm$0.48\end{tabular}          & \begin{tabular}[c]{@{}c@{}}78.00\\ $\pm$0.19\end{tabular}          & \begin{tabular}[c]{@{}c@{}}78.12\\ $\pm$0.50\end{tabular}          & \begin{tabular}[c]{@{}c@{}}78.60\\ $\pm$0.63\end{tabular}          & \begin{tabular}[c]{@{}c@{}}78.21\\ $\pm$0.89\end{tabular}          \\
  \multicolumn{1}{c|}{}&node2vec        & \textbf{\begin{tabular}[c]{@{}c@{}}77.47\\ $\pm$0.20\end{tabular}} & \textbf{\begin{tabular}[c]{@{}c@{}}77.89\\ $\pm$0.16\end{tabular}} & \begin{tabular}[c]{@{}c@{}}78.09*\\ $\pm$0.20\end{tabular}          & \begin{tabular}[c]{@{}c@{}}78.25*\\ $\pm$0.25\end{tabular}          & \begin{tabular}[c]{@{}c@{}}78.53*\\ $\pm$0.44\end{tabular}          & \begin{tabular}[c]{@{}c@{}}78.34*\\ $\pm$0.35\end{tabular}          & \begin{tabular}[c]{@{}c@{}}78.29*\\ $\pm$0.46\end{tabular}          & \begin{tabular}[c]{@{}c@{}}78.81*\\ $\pm$0.61\end{tabular}          & \begin{tabular}[c]{@{}c@{}}78.56*\\ $\pm$0.68\end{tabular}          \\
  \multicolumn{1}{c|}{}&MMDW$\dag$            & -                                                            & -                                                            & -                                                                     & -                                                                     & -                                                                     & -                                                                     & -                                                                     & -                                                                     & -                                                                     \\
  \multicolumn{1}{c|}{}&planetoid & \begin{tabular}[c]{@{}c@{}}39.31$\pm$0\end{tabular}      & \begin{tabular}[c]{@{}c@{}}43.40$\pm$0\end{tabular}      & \begin{tabular}[c]{@{}c@{}}40.50$\pm$0\end{tabular}               & \begin{tabular}[c]{@{}c@{}}40.41$\pm$0\end{tabular}               & \begin{tabular}[c]{@{}c@{}}40.32$\pm$0\end{tabular}               & \begin{tabular}[c]{@{}c@{}}40.44$\pm$0\end{tabular}               & \begin{tabular}[c]{@{}c@{}}40.82$\pm$0\end{tabular}               & \begin{tabular}[c]{@{}c@{}}40.51$\pm$0\end{tabular}               & \begin{tabular}[c]{@{}c@{}}41.13$\pm$0\end{tabular}               \\
  \multicolumn{1}{c|}{}&GCN       & \begin{tabular}[c]{@{}c@{}}39.43\\ $\pm$0.38\end{tabular} & \begin{tabular}[c]{@{}c@{}}39.34\\ $\pm$0.45\end{tabular} & \begin{tabular}[c]{@{}c@{}}39.56\\ $\pm$0.38\end{tabular}          & \begin{tabular}[c]{@{}c@{}}39.46\\ $\pm$0.59\end{tabular}          & \begin{tabular}[c]{@{}c@{}}39.10\\ $\pm$0.45\end{tabular}          & \begin{tabular}[c]{@{}c@{}}39.08\\ $\pm$0.42\end{tabular}          & \begin{tabular}[c]{@{}c@{}}39.31\\ $\pm$0.19\end{tabular}          & \begin{tabular}[c]{@{}c@{}}39.48\\ $\pm$0.84\end{tabular}          & \begin{tabular}[c]{@{}c@{}}39.07\\ $\pm$0.93\end{tabular}          \\
  \multicolumn{1}{c|}{}&\textbf{GESF}   & \begin{tabular}[c]{@{}c@{}}73.19\\ $\pm$0.44\end{tabular} & \begin{tabular}[c]{@{}c@{}}77.70*\\ $\pm$0.64\end{tabular} & \textbf{\begin{tabular}[c]{@{}c@{}}78.47\\ $\pm$0.38\end{tabular}} & \textbf{\begin{tabular}[c]{@{}c@{}}79.63\\ $\pm$0.46\end{tabular}} & \textbf{\begin{tabular}[c]{@{}c@{}}80.23\\ $\pm$0.49\end{tabular}} & \textbf{\begin{tabular}[c]{@{}c@{}}81.05\\ $\pm$0.69\end{tabular}} & \textbf{\begin{tabular}[c]{@{}c@{}}81.53\\ $\pm$0.36\end{tabular}} & \textbf{\begin{tabular}[c]{@{}c@{}}81.77\\ $\pm$0.24\end{tabular}} & \textbf{\begin{tabular}[c]{@{}c@{}}82.62\\ $\pm$0.59\end{tabular}} \\ \hline
  \multicolumn{1}{c|}{\multirow{6}{*}{\rotatebox{90}{Wikipedia}}}&deepwalk       & \begin{tabular}[c]{@{}c@{}}57.44*\\ $\pm$0.67\end{tabular} & \begin{tabular}[c]{@{}c@{}}62.04*\\ $\pm$0.84\end{tabular}          & \begin{tabular}[c]{@{}c@{}}63.15*\\ $\pm$0.77\end{tabular}          & \begin{tabular}[c]{@{}c@{}}64.77*\\ $\pm$0.56\end{tabular}          & \begin{tabular}[c]{@{}c@{}}65.74*\\ $\pm$0.50\end{tabular}          & \begin{tabular}[c]{@{}c@{}}66.63*\\ $\pm$0.84\end{tabular}          & \begin{tabular}[c]{@{}c@{}}65.69\\ $\pm$1.73\end{tabular}          & \begin{tabular}[c]{@{}c@{}}66.61\\ $\pm$1.06\end{tabular}          & \begin{tabular}[c]{@{}c@{}}66.17\\ $\pm$3.04\end{tabular}          \\
  \multicolumn{1}{c|}{}&node2vec       & \textbf{\begin{tabular}[c]{@{}c@{}}57.65\\ $\pm$1.24\end{tabular}} & \begin{tabular}[c]{@{}c@{}}61.73\\ $\pm$0.56\end{tabular}          & \begin{tabular}[c]{@{}c@{}}62.31\\ $\pm$1.42\end{tabular}          & \begin{tabular}[c]{@{}c@{}}64.23\\ $\pm$1.08\end{tabular}          & \begin{tabular}[c]{@{}c@{}}64.94\\ $\pm$0.007\end{tabular}           & \begin{tabular}[c]{@{}c@{}}66.24\\ $\pm$1.06\end{tabular}          & \begin{tabular}[c]{@{}c@{}}65.63\\ $\pm$1.67\end{tabular}          & \begin{tabular}[c]{@{}c@{}}65.99\\ $\pm$1.36\end{tabular}          & \begin{tabular}[c]{@{}c@{}}66.50\\ $\pm$3.66\end{tabular}          \\
  \multicolumn{1}{c|}{}&MMDW           & \begin{tabular}[c]{@{}c@{}}53.05\\ $\pm$0.54\end{tabular} & \begin{tabular}[c]{@{}c@{}}59.45\\ $\pm$0.35\end{tabular}          & \begin{tabular}[c]{@{}c@{}}62.85\\ $\pm$0.55\end{tabular}          & \begin{tabular}[c]{@{}c@{}}62.42\\ $\pm$0.07\end{tabular}          & \begin{tabular}[c]{@{}c@{}}64.26\\ $\pm$1.09\end{tabular}          & \begin{tabular}[c]{@{}c@{}}66.46\\ $\pm$0.85\end{tabular}          & \begin{tabular}[c]{@{}c@{}}67.50*\\ $\pm$0.48\end{tabular}          & \begin{tabular}[c]{@{}c@{}}67.37*\\ $\pm$0.56\end{tabular}          & \begin{tabular}[c]{@{}c@{}}70.20*\\ $\pm$1.22\end{tabular}          \\
  \multicolumn{1}{c|}{}&planetoid      & \begin{tabular}[c]{@{}c@{}}10.02$\pm$0\end{tabular}      & \begin{tabular}[c]{@{}c@{}}12.62$\pm$0\end{tabular}               & \begin{tabular}[c]{@{}c@{}}12.04$\pm$0\end{tabular}               & \begin{tabular}[c]{@{}c@{}}14.06$\pm$0\end{tabular}               & \begin{tabular}[c]{@{}c@{}}16.87$\pm$0\end{tabular}               & \begin{tabular}[c]{@{}c@{}}21.88$\pm$0\end{tabular}               & \begin{tabular}[c]{@{}c@{}}28.39$\pm$0\end{tabular}               & \begin{tabular}[c]{@{}c@{}}14.52$\pm$0\end{tabular}               & \begin{tabular}[c]{@{}c@{}}55.60$\pm$0\end{tabular}               \\
  \multicolumn{1}{c|}{}&GCN            & \begin{tabular}[c]{@{}c@{}}11.01\\ $\pm$0.47\end{tabular} & \begin{tabular}[c]{@{}c@{}}11.19\\ $\pm$0.23\end{tabular}          & \begin{tabular}[c]{@{}c@{}}11.38\\ $\pm$0.48\end{tabular}          & \begin{tabular}[c]{@{}c@{}}11.34\\ $\pm$0.33\end{tabular}          & \begin{tabular}[c]{@{}c@{}}11.47\\ $\pm$0.43\end{tabular}          & \begin{tabular}[c]{@{}c@{}}10.85\\ $\pm$0.84\end{tabular}          & \begin{tabular}[c]{@{}c@{}}10.49\\ $\pm$1.44\end{tabular}          & \begin{tabular}[c]{@{}c@{}}12.20\\ $\pm$1.15\end{tabular}          & \begin{tabular}[c]{@{}c@{}}11.81\\ $\pm$0.24\end{tabular}          \\
  \multicolumn{1}{c|}{}&\textbf{GESF}  & \begin{tabular}[c]{@{}c@{}}55.38\\ $\pm$1.66\end{tabular} & \textbf{\begin{tabular}[c]{@{}c@{}}63.12\\ $\pm$1.07\end{tabular}} & \textbf{\begin{tabular}[c]{@{}c@{}}64.92\\ $\pm$1.23\end{tabular}} & \textbf{\begin{tabular}[c]{@{}c@{}}67.36\\ $\pm$0.48\end{tabular}} & \textbf{\begin{tabular}[c]{@{}c@{}}68.63\\ $\pm$0.65\end{tabular}} & \textbf{\begin{tabular}[c]{@{}c@{}}70.72\\ $\pm$0.77\end{tabular}} & \textbf{\begin{tabular}[c]{@{}c@{}}70.37\\ $\pm$2.59\end{tabular}} & \textbf{\begin{tabular}[c]{@{}c@{}}72.84\\ $\pm$1.56\end{tabular}} & \textbf{\begin{tabular}[c]{@{}c@{}}73.61\\ $\pm$2.14\end{tabular}} \\ \hline
  \multicolumn{1}{c|}{\multirow{6}{*}{\rotatebox{90}{Email-eu}}}&deepwalk       & \begin{tabular}[c]{@{}c@{}}61.11*\\ $\pm$4.71\end{tabular}          & \begin{tabular}[c]{@{}c@{}}67.81*\\ $\pm$2.59\end{tabular}          & \begin{tabular}[c]{@{}c@{}}71.15*\\ $\pm$2.73\end{tabular}          & \begin{tabular}[c]{@{}c@{}}72.54*\\ $\pm$1.79\end{tabular}          & \begin{tabular}[c]{@{}c@{}}75.30*\\ $\pm$0.58\end{tabular}          & \begin{tabular}[c]{@{}c@{}}74.53*\\ $\pm$1.38\end{tabular}          & \begin{tabular}[c]{@{}c@{}}75.48\\ $\pm$1.92\end{tabular}          & \begin{tabular}[c]{@{}c@{}}75.72*\\ $\pm$2.53\end{tabular}          & \begin{tabular}[c]{@{}c@{}}77.60*\\ $\pm$3.56\end{tabular}          \\
  \multicolumn{1}{c|}{}&node2vec       & \begin{tabular}[c]{@{}c@{}}60.60\\ $\pm$4.97\end{tabular}          & \begin{tabular}[c]{@{}c@{}}66.59\\ $\pm$2.81\end{tabular}          & \begin{tabular}[c]{@{}c@{}}69.84\\ $\pm$2.39\end{tabular}          & \begin{tabular}[c]{@{}c@{}}71.91\\ $\pm$1.16\end{tabular}          & \begin{tabular}[c]{@{}c@{}}74.38\\ $\pm$0.64\end{tabular}          & \begin{tabular}[c]{@{}c@{}}74.28\\ $\pm$2.29\end{tabular}          & \begin{tabular}[c]{@{}c@{}}75.81*\\ $\pm$1.76\end{tabular}          & \begin{tabular}[c]{@{}c@{}}75.62\\ $\pm$2.20\end{tabular}          & \begin{tabular}[c]{@{}c@{}}76.00\\ $\pm$2.68\end{tabular}          \\
  \multicolumn{1}{c|}{}&MMDW           & \begin{tabular}[c]{@{}c@{}}36.76\\ $\pm$0.90\end{tabular}          & \begin{tabular}[c]{@{}c@{}}40.72\\ $\pm$2.54\end{tabular}          & \begin{tabular}[c]{@{}c@{}}43.22\\ $\pm$0.61\end{tabular}          & \begin{tabular}[c]{@{}c@{}}43.01\\ $\pm$1.52\end{tabular}          & \begin{tabular}[c]{@{}c@{}}46.11\\ $\pm$1.23\end{tabular}          & \begin{tabular}[c]{@{}c@{}}44.94\\ $\pm$0.38\end{tabular}          & \begin{tabular}[c]{@{}c@{}}48.08\\ $\pm$1.21\end{tabular}          & \begin{tabular}[c]{@{}c@{}}53.62\\ $\pm$0.83\end{tabular}          & \begin{tabular}[c]{@{}c@{}}65.50\\ $\pm$1.34\end{tabular}          \\
  \multicolumn{1}{c|}{}&planetoid      & \begin{tabular}[c]{@{}c@{}}46.47$\pm$0\end{tabular}               & \begin{tabular}[c]{@{}c@{}}58.76$\pm$0\end{tabular}               & \begin{tabular}[c]{@{}c@{}}55.46$\pm$0\end{tabular}               & \begin{tabular}[c]{@{}c@{}}54.30$\pm$0\end{tabular}               & \begin{tabular}[c]{@{}c@{}}52.88$\pm$0\end{tabular}               & \begin{tabular}[c]{@{}c@{}}50.12$\pm$0\end{tabular}               & \begin{tabular}[c]{@{}c@{}}49.67$\pm$0\end{tabular}               & \begin{tabular}[c]{@{}c@{}}55.94$\pm$0\end{tabular}               & \begin{tabular}[c]{@{}c@{}}57.43$\pm$0\end{tabular}               \\
  \multicolumn{1}{c|}{}&GCN            & \begin{tabular}[c]{@{}c@{}}0.63\\ $\pm$0.13\end{tabular}          & \begin{tabular}[c]{@{}c@{}}0.58\\ $\pm$0.07\end{tabular}          & \begin{tabular}[c]{@{}c@{}}0.80\\ $\pm$0.08\end{tabular}          & \begin{tabular}[c]{@{}c@{}}0.66\\ $\pm$0.00\end{tabular}          & \begin{tabular}[c]{@{}c@{}}0.73\\ $\pm$0.41\end{tabular}          & \begin{tabular}[c]{@{}c@{}}0.66\\ $\pm$0.14\end{tabular}          & \begin{tabular}[c]{@{}c@{}}4.58\\ $\pm$5.64\end{tabular}          & \begin{tabular}[c]{@{}c@{}}0.50\\ $\pm$0.00\end{tabular}          & \begin{tabular}[c]{@{}c@{}}0.50\\ $\pm$0.71\end{tabular}          \\
  \multicolumn{1}{c|}{}&\textbf{GESF}  & \textbf{\begin{tabular}[c]{@{}c@{}}64.97\\ $\pm$6.80\end{tabular}} & \textbf{\begin{tabular}[c]{@{}c@{}}68.78\\ $\pm$2.18\end{tabular}} & \textbf{\begin{tabular}[c]{@{}c@{}}72.31\\ $\pm$2.56\end{tabular}} & \textbf{\begin{tabular}[c]{@{}c@{}}73.74\\ $\pm$0.67\end{tabular}} & \textbf{\begin{tabular}[c]{@{}c@{}}75.43\\ $\pm$0.99\end{tabular}} & \textbf{\begin{tabular}[c]{@{}c@{}}76.70\\ $\pm$1.42\end{tabular}} & \textbf{\begin{tabular}[c]{@{}c@{}}77.30\\ $\pm$2.77\end{tabular}} & \textbf{\begin{tabular}[c]{@{}c@{}}79.27\\ $\pm$2.50\end{tabular}} & \textbf{\begin{tabular}[c]{@{}c@{}}81.67\\ $\pm$1.53\end{tabular}} \\ \bottomrule
  
  \end{tabular}
  \\
  $\dag$ MMDW takes over 64GB memory on the experiment of Pubmed. The results of MMDW are not available in this comparison.
  }
  \end{table}

\subsection{Comparison on heterogeneous graphs}
We next conduct evaluation on heterogeneous graphs, where learned node embedding vectors are used for  multi-label classification.


\paragraph{Datasets}
The used datasets include
{\small
\begin{itemize}
\item \textbf{DBLP} \citep{ji2010graph} is an academic community network. Here we obtain a subset of the large network with two types of nodes, authors and key words from authors' publications.
The generated subgraph includes $27K$ (authors) + $3.7K$ (key words) vertexes. 
The link between a pair of author indicates { the coauthor relationships}, and the link between an author and a word means { the word belongs to at least one publication of this author}.
There are 66,832 edges between pairs of authors and 338,210 edges between authors and words. Each node can have multiple labels out of four in total.
\item \textbf{BlogCatalog} \citep{Wang-etal10} is a social media network with 55,814 users and according to the interests of users, they are classified into multiple overlapped groups. We take the five largest groups to evaluate the performance of methods. Users and tags are two types of nodes.
The 5,413 tags are generated by users with their blogs as keywords. Therefore, tags are shared with different users and also have connections since some tags are generated from the same blogs. The number of edges between users, between tags and between users and tags are about 1.4M, 619K and 343K respectively.
\end{itemize}
}

\paragraph{Methods for Comparison}
{\bc
To illustrate the validation of the performance of GESF on heterogeneous graphs, we conduct the experiments on two stages: (1) 
comparing GESF with Deepwalk \citep{perozzi2014deepwalk} and node2vec \citep{grover2016node2vec} on the graphs by treating all nodes as the same type (GESF with $k=1$ in a homogeneous setting);
(2) comparing GESF with the state-of-art heterogeneous graph embedding method, \textit{metapath2vec} \citep{dong2017metapath2vec}, in a heterogeneous setting. The hyper-parameters of the method are fine-tuned and \textit{metapath2vec++} is chosen as the option for the comparison.
  
}

\paragraph{Experiment Setup and Results}

For fair comparison, the dimension of representation vectors is chosen to be the same for all algorithms (the dimension is 64). We fine-tune the hyper-parameter for all of them. The details of GESF for multi-label case are as follows.
{\small
\begin{itemize}
  \item \textbf{Supervised Component:} Since it is a multi-label classification problem, each label can be treated as a binary classification problem. Therefore, we apply logistic regression for each label and for an arbitrary instance $x$ and the $i$-th label $y_i$, the supervised component is formulated as 
  $l(x, y_i) = \log(1 + \exp(w_i^\top x + b_i)) - y_i(w_i^\top x + b_i)$, 
  where $(w_i\in\mathbb{R}^d, b_i\in\mathbb{R})$ is the classifier for the $i$-th label. Therefore, the supervised component in Eq. \eqref{eq:over-all} is defined as $\sum_{v\in \mathcal{V}_{\text{label}}} \sum_i l(x, y_i) + \lambda_w \mathcal{R}(w)$ and $\mathcal{R}(w)$ is the regularization component for $w$, where $\lambda_w$ is chosen to be $10^{-4}$.
  \item \textbf{Unsupervised Embedding Mapping Component:} We design a two-layes NN with a 64-dimensional hidden layer for each type of nodes with the types of nodes in its neighborhood to formulate the mapping from embedding of neighbors to the embedding of the target node. We also form a two-layer 1-to-1 NN wth a 3 dimensional hidden layer to construct the matrix function for the adjacency matrix $A$ for the whole graph. We pre-process the matrix $A$ with an eigenvalue decomposition by preserving the highest 1000 eigenvalues in default. We denote the nodes to be classified as type 1 and the other type as type 2. The balance hyper-parameter $[\lambda_1, \lambda_2]$ is set to be [0.2, 200]. 
\end{itemize}
}

For the datasets DBLP and BlogCatalog, we carry out the experiments on each of them and compare the performance among all methods mentioned above. Since it is a multi-label classification task, we take \textit{f1-score(macro, micro)} as the evaluation score for the comparison. The percentage of labeled samples is chosen from 10\% to 90\%, while the remaining samples are used for evaluation. We repeat all experiments for three times and report the mean and standard deviation of their performance in the Tables \ref{tab:macro-dblp}. 
We can observe that in most cases, {\pc
GESF in heterogeneous setting has the best performance, while 
GESF in homogeneous setting achieves the second best results, demonstrating the validity of our proposed universal graph embedding mechanism.}

\begin{table}[!htp]
  \centering
  \caption{F1-score (macro, micro) (\%) of Multi-label Classification Experiments}
  \label{tab:macro-dblp}
  {\scriptsize
  \begin{tabular}{@{}ccccccccccc@{}}
  \toprule
  \multicolumn{1}{c}{}& training\% & 10.00\%                                                               & 20.00\%                                                               & 30.00\%                                                               & 40.00\%                                                               & 50.00\%                                                               & 60.00\%                                                               & 70.00\%                                                               & 80.00\%                                                               & 90.00\%                                                               \\ \midrule
  \multicolumn{1}{c|}{\multirow{5}{*}{
    \rotatebox{90}{
    \begin{tabular}[c]{@{}c@{}}
      DBLP (macro)
    \end{tabular}
    }
  }}&Deepwalk       & \begin{tabular}[c]{@{}c@{}}74.48*\\ $\pm$0.34\end{tabular}          & \begin{tabular}[c]{@{}c@{}}74.87\\ $\pm$0.14\end{tabular}          & \begin{tabular}[c]{@{}c@{}}74.95\\ $\pm$0.15\end{tabular}          & \begin{tabular}[c]{@{}c@{}}75.10\\ $\pm$0.22\end{tabular}          & \begin{tabular}[c]{@{}c@{}}75.07\\ $\pm$0.22\end{tabular}          & \begin{tabular}[c]{@{}c@{}}75.44\\ $\pm$0.22\end{tabular}          & \begin{tabular}[c]{@{}c@{}}75.33\\ $\pm$0.33\end{tabular}          & \begin{tabular}[c]{@{}c@{}}74.75\\ $\pm$0.31\end{tabular}          & \begin{tabular}[c]{@{}c@{}}75.36\\ $\pm$0.73\end{tabular}          \\
  \multicolumn{1}{c|}{}&Node2vec       & \begin{tabular}[c]{@{}c@{}}73.37\\ $\pm$0.24\end{tabular}          & \begin{tabular}[c]{@{}c@{}}73.94\\ $\pm$0.11\end{tabular}          & \begin{tabular}[c]{@{}c@{}}74.00\\ $\pm$0.18\end{tabular}          & \begin{tabular}[c]{@{}c@{}}74.25\\ $\pm$0.23\end{tabular}          & \begin{tabular}[c]{@{}c@{}}74.06\\ $\pm$0.31\end{tabular}          & \begin{tabular}[c]{@{}c@{}}74.52\\ $\pm$0.21\end{tabular}          & \begin{tabular}[c]{@{}c@{}}74.52\\ $\pm$0.35\end{tabular}          & \begin{tabular}[c]{@{}c@{}}74.32\\ $\pm$0.26\end{tabular}          & \begin{tabular}[c]{@{}c@{}}74.55\\ $\pm$0.57\end{tabular}          \\
  \multicolumn{1}{c|}{}&Metapath2vec++ & \textbf{\begin{tabular}[c]{@{}c@{}}74.82\\ $\pm$0.30\end{tabular}} & \begin{tabular}[c]{@{}c@{}}75.27\\ $\pm$0.08\end{tabular}          & \begin{tabular}[c]{@{}c@{}}75.55\\ $\pm$0.12\end{tabular}          & \begin{tabular}[c]{@{}c@{}}75.63\\ $\pm$0.27\end{tabular}          & \begin{tabular}[c]{@{}c@{}}75.53\\ $\pm$0.22\end{tabular}          & \begin{tabular}[c]{@{}c@{}}75.92\\ $\pm$0.24\end{tabular}          & \begin{tabular}[c]{@{}c@{}}75.92\\ $\pm$0.42\end{tabular}          & \begin{tabular}[c]{@{}c@{}}75.56\\ $\pm$0.36\end{tabular}          & \begin{tabular}[c]{@{}c@{}}76.09\\ $\pm$0.46\end{tabular}          \\
  \multicolumn{1}{c|}{}&\begin{tabular}[c]{c@{}c@{}}\textbf{GESF}\\(Homogeneous)\end{tabular}  & \begin{tabular}[c]{@{}c@{}}72.51\\ $\pm$0.91\end{tabular}          & \begin{tabular}[c]{@{}c@{}}76.89*\\ $\pm$0.06\end{tabular}          & \begin{tabular}[c]{@{}c@{}}79.91*\\ $\pm$0.09\end{tabular}          & \begin{tabular}[c]{@{}c@{}}82.14*\\ $\pm$0.26\end{tabular}          & \begin{tabular}[c]{@{}c@{}}84.60*\\ $\pm$0.42\end{tabular}          & \begin{tabular}[c]{@{}c@{}}86.34*\\ $\pm$0.52\end{tabular}          & \begin{tabular}[c]{@{}c@{}}87.49*\\ $\pm$0.31\end{tabular}          & \begin{tabular}[c]{@{}c@{}}88.21*\\ $\pm$0.25\end{tabular}          & \begin{tabular}[c]{@{}c@{}}89.61*\\ $\pm$0.58\end{tabular}          \\
  \multicolumn{1}{c|}{}&\begin{tabular}[c]{c@{}c@{}}\textbf{GESF}\\(Heterogeneous)\end{tabular} & \begin{tabular}[c]{@{}c@{}}74.06\\ $\pm$0.37\end{tabular}          & \textbf{\begin{tabular}[c]{@{}c@{}}78.43\\ $\pm$0.61\end{tabular}} & \textbf{\begin{tabular}[c]{@{}c@{}}81.00\\ $\pm$0.19\end{tabular}} & \textbf{\begin{tabular}[c]{@{}c@{}}83.18\\ $\pm$0.16\end{tabular}} & \textbf{\begin{tabular}[c]{@{}c@{}}84.95\\ $\pm$0.22\end{tabular}} & \textbf{\begin{tabular}[c]{@{}c@{}}86.91\\ $\pm$0.54\end{tabular}} & \textbf{\begin{tabular}[c]{@{}c@{}}88.30\\ $\pm$0.29\end{tabular}} & \textbf{\begin{tabular}[c]{@{}c@{}}89.18\\ $\pm$0.17\end{tabular}} & \textbf{\begin{tabular}[c]{@{}c@{}}90.37\\ $\pm$0.38\end{tabular}} \\ \hline
    \multicolumn{1}{c|}{\multirow{5}{*}{
      \rotatebox{90}{
      \begin{tabular}[c]{@{}c@{}}
        DBLP (micro)
      \end{tabular}
      }
    }}&
  deepwalk       & \begin{tabular}[c]{@{}c@{}}76.65\\ $\pm$0.25\end{tabular}          & \begin{tabular}[c]{@{}c@{}}77.03\\ $\pm$0.19\end{tabular}          & \begin{tabular}[c]{@{}c@{}}77.15\\ $\pm$0.15\end{tabular}          & \begin{tabular}[c]{@{}c@{}}77.21\\ $\pm$0.15\end{tabular}          & \begin{tabular}[c]{@{}c@{}}77.20\\ $\pm$0.17\end{tabular}          & \begin{tabular}[c]{@{}c@{}}77.60\\ $\pm$0.20\end{tabular}          & \begin{tabular}[c]{@{}c@{}}77.44\\ $\pm$0.31\end{tabular}          & \begin{tabular}[c]{@{}c@{}}76.87\\ $\pm$0.35\end{tabular}          & \begin{tabular}[c]{@{}c@{}}77.54\\ $\pm$0.72\end{tabular}          \\
  \multicolumn{1}{c|}{}&node2vec       & \begin{tabular}[c]{@{}c@{}}75.65\\ $\pm$0.16\end{tabular}          & \begin{tabular}[c]{@{}c@{}}76.21\\ $\pm$0.12\end{tabular}          & \begin{tabular}[c]{@{}c@{}}76.30\\ $\pm$0.16\end{tabular}          & \begin{tabular}[c]{@{}c@{}}76.48\\ $\pm$0.18\end{tabular}          & \begin{tabular}[c]{@{}c@{}}76.33\\ $\pm$0.25\end{tabular}          & \begin{tabular}[c]{@{}c@{}}76.82\\ $\pm$0.22\end{tabular}          & \begin{tabular}[c]{@{}c@{}}76.76\\ $\pm$0.33\end{tabular}          & \begin{tabular}[c]{@{}c@{}}76.44\\ $\pm$0.34\end{tabular}          & \begin{tabular}[c]{@{}c@{}}76.73\\ $\pm$0.55\end{tabular}          \\
  \multicolumn{1}{c|}{}&metapath2vec++ & \begin{tabular}[c]{@{}c@{}}76.98*\\ $\pm$0.21\end{tabular}          & \begin{tabular}[c]{@{}c@{}}77.38\\ $\pm$0.10\end{tabular}          & \begin{tabular}[c]{@{}c@{}}77.66\\ $\pm$0.08\end{tabular}          & \begin{tabular}[c]{@{}c@{}}77.70\\ $\pm$0.21\end{tabular}          & \begin{tabular}[c]{@{}c@{}}77.61\\ $\pm$0.18\end{tabular}          & \begin{tabular}[c]{@{}c@{}}78.04\\ $\pm$0.18\end{tabular}          & \begin{tabular}[c]{@{}c@{}}77.95\\ $\pm$0.39\end{tabular}          & \begin{tabular}[c]{@{}c@{}}77.54\\ $\pm$0.36\end{tabular}          & \begin{tabular}[c]{@{}c@{}}78.02\\ $\pm$0.47\end{tabular}          \\
  \multicolumn{1}{c|}{}&\begin{tabular}[c]{c@{}c@{}}\textbf{GESF}\\(Homogeneous)\end{tabular}  & \begin{tabular}[c]{@{}c@{}}74.37\\ $\pm$0.88\end{tabular}          & \begin{tabular}[c]{@{}c@{}}78.52*\\ $\pm$0.09\end{tabular}          & \begin{tabular}[c]{@{}c@{}}81.47*\\ $\pm$0.11\end{tabular}          & \begin{tabular}[c]{@{}c@{}}83.56*\\ $\pm$0.28\end{tabular}          & \begin{tabular}[c]{@{}c@{}}85.81*\\ $\pm$0.42\end{tabular}          & \begin{tabular}[c]{@{}c@{}}87.51*\\ $\pm$0.54\end{tabular}          & \begin{tabular}[c]{@{}c@{}}88.44*\\ $\pm$0.27\end{tabular}          & \begin{tabular}[c]{@{}c@{}}89.16*\\ $\pm$0.22\end{tabular}          & \begin{tabular}[c]{@{}c@{}}90.54*\\ $\pm$0.50\end{tabular}          \\
  \multicolumn{1}{c|}{}&\begin{tabular}[c]{c@{}c@{}}\textbf{GESF}\\(Heterogeneous)\end{tabular} & \textbf{\begin{tabular}[c]{@{}c@{}}77.06\\ $\pm$0.29\end{tabular}} & \textbf{\begin{tabular}[c]{@{}c@{}}80.67\\ $\pm$0.45\end{tabular}} & \textbf{\begin{tabular}[c]{@{}c@{}}82.87\\ $\pm$0.13\end{tabular}} & \textbf{\begin{tabular}[c]{@{}c@{}}84.75\\ $\pm$0.16\end{tabular}} & \textbf{\begin{tabular}[c]{@{}c@{}}86.29\\ $\pm$0.22\end{tabular}} & \textbf{\begin{tabular}[c]{@{}c@{}}88.09\\ $\pm$0.47\end{tabular}} & \textbf{\begin{tabular}[c]{@{}c@{}}89.27\\ $\pm$0.25\end{tabular}} & \textbf{\begin{tabular}[c]{@{}c@{}}90.11\\ $\pm$0.13\end{tabular}} & \textbf{\begin{tabular}[c]{@{}c@{}}91.22\\ $\pm$0.43\end{tabular}} \\ \hline
    \multicolumn{1}{c|}{\multirow{5}{*}{
      \rotatebox{90}{
      \begin{tabular}[c]{@{}c@{}}
        BlogCatalog (macro)
      \end{tabular}
      }
    }}&
  deepwalk       & \begin{tabular}[c]{@{}c@{}}45.13\\ $\pm$0.68\end{tabular}          & \begin{tabular}[c]{@{}c@{}}44.64\\ $\pm$0.21\end{tabular}          & \begin{tabular}[c]{@{}c@{}}44.52\\ $\pm$0.42\end{tabular}          & \begin{tabular}[c]{@{}c@{}}44.64\\ $\pm$0.23\end{tabular}          & \begin{tabular}[c]{@{}c@{}}44.32\\ $\pm$0.21\end{tabular}          & \begin{tabular}[c]{@{}c@{}}44.36\\ $\pm$0.46\end{tabular}          & \begin{tabular}[c]{@{}c@{}}44.78\\ $\pm$0.48\end{tabular}          & \begin{tabular}[c]{@{}c@{}}44.33\\ $\pm$0.62\end{tabular}          & \begin{tabular}[c]{@{}c@{}}44.49\\ $\pm$0.86\end{tabular}          \\
  \multicolumn{1}{c|}{}&node2vec       & \begin{tabular}[c]{@{}c@{}}45.78\\ $\pm$0.68\end{tabular}          & \begin{tabular}[c]{@{}c@{}}45.42\\ $\pm$0.30\end{tabular}          & \begin{tabular}[c]{@{}c@{}}45.28\\ $\pm$0.32\end{tabular}          & \begin{tabular}[c]{@{}c@{}}45.41\\ $\pm$0.18\end{tabular}          & \begin{tabular}[c]{@{}c@{}}45.17\\ $\pm$0.20\end{tabular}          & \begin{tabular}[c]{@{}c@{}}45.19\\ $\pm$0.36\end{tabular}          & \begin{tabular}[c]{@{}c@{}}45.57\\ $\pm$0.13\end{tabular}          & \begin{tabular}[c]{@{}c@{}}45.04\\ $\pm$0.31\end{tabular}          & \begin{tabular}[c]{@{}c@{}}44.96\\ $\pm$0.55\end{tabular}          \\
  \multicolumn{1}{c|}{}&metapath2vec++ & \begin{tabular}[c]{@{}c@{}}37.46\\ $\pm$0.61\end{tabular}          & \begin{tabular}[c]{@{}c@{}}36.72\\ $\pm$0.36\end{tabular}          & \begin{tabular}[c]{@{}c@{}}36.69\\ $\pm$0.29\end{tabular}          & \begin{tabular}[c]{@{}c@{}}36.58\\ $\pm$0.28\end{tabular}          & \begin{tabular}[c]{@{}c@{}}36.74\\ $\pm$0.17\end{tabular}          & \begin{tabular}[c]{@{}c@{}}36.90\\ $\pm$0.33\end{tabular}          & \begin{tabular}[c]{@{}c@{}}36.89\\ $\pm$0.32\end{tabular}          & \begin{tabular}[c]{@{}c@{}}36.42\\ $\pm$0.41\end{tabular}          & \begin{tabular}[c]{@{}c@{}}36.16\\ $\pm$0.98\end{tabular}          \\
  \multicolumn{1}{c|}{}&\begin{tabular}[c]{c@{}c@{}}\textbf{GESF}\\(Homogeneous)\end{tabular}  & \begin{tabular}[c]{@{}c@{}}47.63*\\ $\pm$3.16\end{tabular}          & \begin{tabular}[c]{@{}c@{}}50.99*\\ $\pm$0.09\end{tabular}          & \begin{tabular}[c]{@{}c@{}}51.70*\\ $\pm$0.19\end{tabular}          & \begin{tabular}[c]{@{}c@{}}50.04*\\ $\pm$1.90\end{tabular}          & \begin{tabular}[c]{@{}c@{}}50.60*\\ $\pm$1.73\end{tabular}          & \begin{tabular}[c]{@{}c@{}}50.34*\\ $\pm$0.55\end{tabular}          & \begin{tabular}[c]{@{}c@{}}52.20*\\ $\pm$1.11\end{tabular}          & \begin{tabular}[c]{@{}c@{}}51.88*\\ $\pm$1.08\end{tabular}          & \begin{tabular}[c]{@{}c@{}}51.36*\\ $\pm$0.26\end{tabular}          \\
  \multicolumn{1}{c|}{}&\begin{tabular}[c]{c@{}c@{}}\textbf{GESF}\\(Heterogeneous)\end{tabular} & \textbf{\begin{tabular}[c]{@{}c@{}}49.65\\ $\pm$0.63\end{tabular}} & \textbf{\begin{tabular}[c]{@{}c@{}}51.47\\ $\pm$0.40\end{tabular}} & \textbf{\begin{tabular}[c]{@{}c@{}}52.69\\ $\pm$0.24\end{tabular}} & \textbf{\begin{tabular}[c]{@{}c@{}}53.37\\ $\pm$0.45\end{tabular}} & \textbf{\begin{tabular}[c]{@{}c@{}}53.73\\ $\pm$0.01\end{tabular}} & \textbf{\begin{tabular}[c]{@{}c@{}}53.97\\ $\pm$0.43\end{tabular}} & \textbf{\begin{tabular}[c]{@{}c@{}}53.83\\ $\pm$0.62\end{tabular}} & \textbf{\begin{tabular}[c]{@{}c@{}}54.07\\ $\pm$0.40\end{tabular}} & \textbf{\begin{tabular}[c]{@{}c@{}}53.36\\ $\pm$0.84\end{tabular}} \\ \bottomrule
    \multicolumn{1}{c|}{\multirow{5}{*}{
      \rotatebox{90}{
      \begin{tabular}[c]{@{}c@{}}
        BlogCatalog (micro)
      \end{tabular}
      }
    }}&
  deepwalk       & \begin{tabular}[c]{@{}c@{}}47.93\\ $\pm$0.48\end{tabular}          & \begin{tabular}[c]{@{}c@{}}47.36\\ $\pm$0.15\end{tabular}          & \begin{tabular}[c]{@{}c@{}}47.25\\ $\pm$0.45\end{tabular}          & \begin{tabular}[c]{@{}c@{}}47.30\\ $\pm$0.19\end{tabular}          & \begin{tabular}[c]{@{}c@{}}47.07\\ $\pm$0.16\end{tabular}          & \begin{tabular}[c]{@{}c@{}}47.09\\ $\pm$0.48\end{tabular}          & \begin{tabular}[c]{@{}c@{}}47.39\\ $\pm$0.49\end{tabular}          & \begin{tabular}[c]{@{}c@{}}47.02\\ $\pm$0.72\end{tabular}          & \begin{tabular}[c]{@{}c@{}}47.27\\ $\pm$0.82\end{tabular}          \\
  \multicolumn{1}{c|}{}&node2vec       & \begin{tabular}[c]{@{}c@{}}48.52\\ $\pm$0.50\end{tabular}          & \begin{tabular}[c]{@{}c@{}}48.20\\ $\pm$0.18\end{tabular}          & \begin{tabular}[c]{@{}c@{}}48.01\\ $\pm$0.31\end{tabular}          & \begin{tabular}[c]{@{}c@{}}48.18\\ $\pm$0.17\end{tabular}          & \begin{tabular}[c]{@{}c@{}}47.97\\ $\pm$0.25\end{tabular}          & \begin{tabular}[c]{@{}c@{}}48.04\\ $\pm$0.39\end{tabular}          & \begin{tabular}[c]{@{}c@{}}48.22\\ $\pm$0.15\end{tabular}          & \begin{tabular}[c]{@{}c@{}}47.83\\ $\pm$0.32\end{tabular}          & \begin{tabular}[c]{@{}c@{}}47.95\\ $\pm$0.55\end{tabular}          \\
  \multicolumn{1}{c|}{}&metapath2vec++ & \begin{tabular}[c]{@{}c@{}}40.90\\ $\pm$0.34\end{tabular}          & \begin{tabular}[c]{@{}c@{}}40.06\\ $\pm$0.21\end{tabular}          & \begin{tabular}[c]{@{}c@{}}40.10\\ $\pm$0.25\end{tabular}          & \begin{tabular}[c]{@{}c@{}}39.97\\ $\pm$0.22\end{tabular}          & \begin{tabular}[c]{@{}c@{}}40.04\\ $\pm$0.17\end{tabular}          & \begin{tabular}[c]{@{}c@{}}40.22\\ $\pm$0.29\end{tabular}          & \begin{tabular}[c]{@{}c@{}}40.21\\ $\pm$0.22\end{tabular}          & \begin{tabular}[c]{@{}c@{}}39.82\\ $\pm$0.44\end{tabular}          & \begin{tabular}[c]{@{}c@{}}39.66\\ $\pm$0.82\end{tabular}          \\
  \multicolumn{1}{c|}{}&\begin{tabular}[c]{c@{}c@{}}\textbf{GESF}\\(Homogeneous)\end{tabular}  & \begin{tabular}[c]{@{}c@{}}50.75*\\ $\pm$2.74\end{tabular}          & \begin{tabular}[c]{@{}c@{}}54.26*\\ $\pm$0.12\end{tabular}          & \begin{tabular}[c]{@{}c@{}}55.08*\\ $\pm$0.08\end{tabular}          & \begin{tabular}[c]{@{}c@{}}53.41*\\ $\pm$1.92\end{tabular}          & \begin{tabular}[c]{@{}c@{}}53.88*\\ $\pm$1.70\end{tabular}          & \begin{tabular}[c]{@{}c@{}}53.57*\\ $\pm$0.50\end{tabular}          & \begin{tabular}[c]{@{}c@{}}55.36*\\ $\pm$1.30\end{tabular}          & \begin{tabular}[c]{@{}c@{}}54.93*\\ $\pm$1.07\end{tabular}          & \begin{tabular}[c]{@{}c@{}}55.03*\\ $\pm$0.52\end{tabular}          \\
  \multicolumn{1}{c|}{}&\begin{tabular}[c]{c@{}c@{}}\textbf{GESF}\\(Heterogeneous)\end{tabular} & \textbf{\begin{tabular}[c]{@{}c@{}}53.06\\ $\pm$0.45\end{tabular}} & \textbf{\begin{tabular}[c]{@{}c@{}}54.77\\ $\pm$0.21\end{tabular}} & \textbf{\begin{tabular}[c]{@{}c@{}}55.93\\ $\pm$0.17\end{tabular}} & \textbf{\begin{tabular}[c]{@{}c@{}}56.41\\ $\pm$0.31\end{tabular}} & \textbf{\begin{tabular}[c]{@{}c@{}}56.86\\ $\pm$0.06\end{tabular}} & \textbf{\begin{tabular}[c]{@{}c@{}}57.28\\ $\pm$0.38\end{tabular}} & \textbf{\begin{tabular}[c]{@{}c@{}}57.13\\ $\pm$0.43\end{tabular}} & \textbf{\begin{tabular}[c]{@{}c@{}}57.43\\ $\pm$0.14\end{tabular}} & \textbf{\begin{tabular}[c]{@{}c@{}}56.98\\ $\pm$0.53\end{tabular}} \\ \bottomrule
  
  \end{tabular}
  }
  \end{table}

\section{Conclusion and Future Work}
{\bc
To summarize the whole work, GESF is proposed for a most general graph embedding solution with a theoretical guarantee for the effectiveness of the whole model and impressive experiment results compared to the state-of-art algorithms. For the future work, our model can be extended to more general case e.g, involving the node content or attributes into the embedding learning. One possible solution is to introduce the attributes as a special type of neighbors in the graph and we can utilize multiple set functions to map the embeddings within a more complex heterogeneous graph structure. 
}


{
\bibliographystyle{abbrvnat} 
\bibliography{sample} 
}

\newpage
 \begin{center}
{\Huge \bf
Supplemental Material
}
\end{center}
\vspace{10mm}

We provide the proof to Theorem~\ref{theorem6} in the supplemental material. {\rc The proof in our paper borrows some idea from the proof sketch for a special case of our theorem in \citep{zaheer2017deep}. While \citet{zaheer2017deep} only provides a small paragraph to explain the raw idea on how to prove a special case, we provide complete and rigorous proof and nontrivial extension to the more general case.}

\begin{definition}[Power Sum Symmetric Polynomials] For every integer $k \geq 0$, the $k$-th power sum symmetric polynomial in variables $x_1,x_2,...,x_n$ is defined as
\[
p_k(x_1,x_2,...,x_n) = \sum_{i = 1}^n x_i^k.
\]

\end{definition}

\begin{definition}[Monomial Symmetric Polynomials] For $\boldsymbol{\lambda} = [\lambda_1,\lambda_2,...,\lambda_n]^\top \in \mathbb{R}^n$ where $\lambda_1 \geq \lambda_2 \geq ...\geq \lambda_n \geq 0$, let $\Omega_{\boldsymbol{\lambda}}$ be the set of all permutations of the entries in $\boldsymbol{\lambda}$. For $\boldsymbol{\lambda}$, the monomial symmetric polynomials is defined as
\[
m^{\boldsymbol{\lambda}} = \sum_{\boldsymbol{[\alpha_1,...,\alpha_n]^\top \in \Omega_{\boldsymbol{\lambda}}}} x^{\alpha_1}_1 x^{\alpha_2}_{2}...x^{\alpha_n}_{n}
\]
\end{definition}

\begin{lemma}\label{lemma:approx}
If $f: X \rightarrow\mathbb{R} $ 
is a continuous real-valued function defined on the compact set $X \subset \mathbb{R}^n$, then $\forall \epsilon > 0$, there $\exists$ a polynomial function $p: X \rightarrow \mathbb{R}$, such that $\sup_{\x \in X} | f (\x) - p(\x) | < \epsilon$.
\end{lemma}

\begin{proof}
This lemma is a direct application of Stone-Weierstrass Theorem~\cite{stone1937applications,stone1948generalized} for the compact Hausdorff space. 
\end{proof}

\begin{corollary}\label{corollary:approx}
For any continuous real-valued function $f: X \rightarrow \mathbb{R}$ of the form defined in Theorem~\ref{theorem6}, 
\[
f(\underbrace{x_{1,1}, x_{1,2}, \cdots, x_{1,N_1}}_{G_1}, \underbrace{x_{2,1}, x_{2,2}, \cdots, x_{2,N_2}}_{G_2}, \cdots, \underbrace{x_{K,1}, x_{K,2}, \cdots, x_{K,N_K}}_{G_K}),
\]
there exist polynomial functions $p: X \rightarrow \mathbb{R}$ also permutation invariant within each $G_k$ to closely approximate $f$.
\end{corollary}
\begin{proof}
Let $S^k$ be the set of all possible permutations of $1, ..., N_k$ for the group $G^k$, where $k = 1,...,K$. Suppose $\boldsymbol{\sigma}^k = [\sigma_1^k,\sigma_2^k,...,\sigma^k_{N_k}] \in S^k$ is some permutation for $G^k$. Then we know that $\forall \boldsymbol{\sigma}^k \in S^k$ with $k = 1,...,K$, we have
\[
f(\x_1, \x_2, \cdots, \x_K ) \\ = f([\x_1]_{\boldsymbol{\sigma}^1},[\x_2]_{\boldsymbol{\sigma}^2}, \cdots, [\x_K]_{\boldsymbol{\sigma}^K})
\]
where we let $\x_k$ to denote $[x_{k,1}, x_{k,2}, \cdots, x_{k,N_k}]$ for $G_k$ and $[\x_k]_{\boldsymbol{\sigma}^k}$ to denote the permuted $\x_k$ for simplicity. There are in total $\prod_{k=1}^K (N_k!)$ permutations. 

By lemma~\ref{lemma:approx}, we can know that $\forall \epsilon > 0$, there $\exists$ a polynomial function $q:X \rightarrow \mathbb{R}$ such that $\sup|f(\x_1, \x_2, \cdots, \x_K )-q(\x_1, \x_2, \cdots, \x_K )| \leq \epsilon$. This further implies that $\forall \boldsymbol{\sigma}^k \in S^k$ with $k = 1,...,K$, we have 
\[
\sup|f([\x_1]_{\boldsymbol{\sigma}^1},[\x_2]_{\boldsymbol{\sigma}^2}, \cdots, [\x_K]_{\boldsymbol{\sigma}^K})-q([\x_1]_{\boldsymbol{\sigma}^1},[\x_2]_{\boldsymbol{\sigma}^2}, \cdots, [\x_K]_{\boldsymbol{\sigma}^K})| \leq \epsilon
\]

We let $p(\x_1, \x_2, \cdots, \x_K ) = \frac{1}{\prod_{k=1}^K (N_k!)} \sum_{\boldsymbol{\sigma}^1 \in S^1, \cdots, \boldsymbol{\sigma}^K \in S^K} q([\x_1]_{\boldsymbol{\sigma}^1},[\x_2]_{\boldsymbol{\sigma}^2}, \cdots, [\x_K]_{\boldsymbol{\sigma}^K})$ and by the property of permutation invariant within each $G^k$ of the function $f$, we have that
\small
\begin{align*}
& \sup|f(\x_1, \x_2, \cdots, \x_K )-p(\x_1, \x_2, \cdots, \x_K )| \\
= & \sup \left|f(\x_1, \x_2, \cdots, \x_K )-\frac{1}{\prod_{k=1}^K (N_k!)} \sum_{\boldsymbol{\sigma}^1 \in S^1, \cdots, \boldsymbol{\sigma}^K \in S^K} q([\x_1]_{\boldsymbol{\sigma}^1},[\x_2]_{\boldsymbol{\sigma}^2}, \cdots, [\x_K]_{\boldsymbol{\sigma}^K})\right| \\
= & \sup \Big|\frac{1}{\prod_{k=1}^K (N_k!)} \sum_{\boldsymbol{\sigma}^1 \in S^1, \cdots, \boldsymbol{\sigma}^K \in S^K} f([\x_1]_{\boldsymbol{\sigma}^1}, \cdots, [\x_K]_{\boldsymbol{\sigma}^K})-\frac{1}{\prod_{k=1}^K (N_k!)} \sum_{\boldsymbol{\sigma}^1 \in S^1, \cdots, \boldsymbol{\sigma}^K \in S^K} q([\x_1]_{\boldsymbol{\sigma}^1}, \cdots, [\x_K]_{\boldsymbol{\sigma}^K}) \Big|\\
\leq  & \frac{1}{\prod_{k=1}^K (N_k!)} \sum_{\boldsymbol{\sigma}^1 \in S^1, \cdots, \boldsymbol{\sigma}^K \in S^K}  \sup \Big | f([\x_1]_{\boldsymbol{\sigma}^1},[\x_2]_{\boldsymbol{\sigma}^2}, \cdots, [\x_K]_{\boldsymbol{\sigma}^K})-q([\x_1]_{\boldsymbol{\sigma}^1},[\x_2]_{\boldsymbol{\sigma}^2}, \cdots, [\x_K]_{\boldsymbol{\sigma}^K}) \Big |\\
\leq  & \frac{1}{\prod_{k=1}^K (N_k!)} \sum_{\boldsymbol{\sigma}^1 \in S^1, \cdots, \boldsymbol{\sigma}^K \in S^K}  \epsilon \\
=& \epsilon
\end{align*}
\normalsize
Therefore, we can see that the polynomial function $p(\x_1, \x_2, \cdots, \x_K )$ is permutation invariant within each group and can closely approximate the function $f(\x_1, \x_2, \cdots, \x_K)$.

\end{proof}

\begin{lemma}\label{lemma:repre} Suppose $p: X \rightarrow \mathbb{R}$ is a real-valued polynomial in the form 
\[
p(\underbrace{x_{1,1}, x_{1,2}, \cdots, x_{1,N_1}}_{G_1}, \cdots, \underbrace{x_{K,1}, x_{K,2}, \cdots, x_{K,N_K}}_{G_K}),
\]
which is permutation invariant within each group $G^k$, $k = 1,...,K$. Then, $p$ can be represented as
\begin{align*}
p(\underbrace{x_{1,1}, x_{1,2}, \cdots, x_{1,N_1}}_{G_1}, \cdots, \underbrace{x_{K,1}, x_{K,2}, \cdots, x_{K,N_K}}_{G_K}) = \sum_{i=1} c_i m_1^{\boldsymbol{\lambda}^i_{1}} m_2^{\boldsymbol{\lambda}^i_{2}} ... m_{K}^{\boldsymbol{\lambda}^i_{K}}
\end{align*}
where $c_i$ is the rational coefficient, $m_k^{\boldsymbol{\lambda}^i_{K}}$ denotes a monomial symmetric polynomial of variables $x_{k,1}, x_{k,2}, \cdots, x_{k,N_k}$ within group $G_k$ and $\boldsymbol{\lambda}^i_{K} \in \mathbb{R}^n$ are some certain exponents for the monomial symmetric polynomial.
\end{lemma}

\begin{proof}
Suppose that the polynomial is expressed as a summation of monomials 
\[
p(\underbrace{x_{1,1}, x_{1,2}, \cdots, x_{1,N_1}}_{G_1}, \cdots, \underbrace{x_{K,1}, x_{K,2}, \cdots, x_{K,N_K}}_{G_K}) = \sum_{r} c_r \prod_{j=1}^{N_1} x^{\alpha^r_{1,j}}_{1,j} \prod_{j=1}^{N_2} x^{\alpha^r_{2,j}}_{2,j},\cdots,\prod_{j = 1}^{N_K} x^{\alpha^r_{K,j}}_{K,j}
\]
where $\alpha^r_{k,j}$ is the exponent of $x_{k,j}$ for the $r$-th monomial.

We consider the term $c_r \prod_{j=1}^{N_1} x^{\alpha^r_{1,j}}_{1,j} \prod_{j=1}^{N_2} x^{\alpha^r_{2,j}}_{2,j}\cdots\prod_{j = 1}^{N_K} x^{\alpha^r_{K,j}}_{K,j}$ for a certain $r$ in the above summation. Since $p(\cdot)$ is partially symmetric within the group $G_1$, then there must exist terms having exponents of permuted $\alpha^r_{1,j}, \forall j = 1,...,N_1$ on $x_{1,j}$ while the other factors $c_r \prod_{j=1}^{N_2} x^{\alpha^r_{2,j}}_{2,j}\cdots\prod_{j = 1}^{N_K} x^{\alpha^r_{K,j}}_{K,j}$ remain the same. (Otherwise, the polynomial is not permutation-invariant w.r.t. group $G_1$.) By summing up those terms together, we get a term with a factor of monomial symmetric polynomial as $c_r m_1^{\boldsymbol{\alpha}_1^r} \prod_{j=1}^{N_2} x^{\alpha^r_{2,j}}_{2,j}\cdots\prod_{j = 1}^{N_K} x^{\alpha^r_{K,j}}_{K,j}$. Based on this term, we consider that the polynomial $p(\cdot)$ is also permuted invariant within $G_2$, which implies that there must exist terms having exponents of permuted $\alpha^r_{2,j}, \forall j=1,...,N_2$ on $x_{2,j}$ while the other factors $c_r m_1^{\boldsymbol{\alpha}_1^r} \prod_{j=1}^{N_3} x^{\alpha^r_{3,j}}_{3,j}\cdots\prod_{j = 1}^{N_K} x^{\alpha^r_{K,j}}_{K,j}$ remain the same. Therefore, adding up all those terms together, we have $c_r m_1^{\boldsymbol{\alpha}_1^r} m_2^{\boldsymbol{\alpha}_2^r} \prod_{j=1}^{N_3} x^{\alpha^r_{3,j}}_{3,j}\cdots\prod_{j = 1}^{N_K} x^{\alpha^r_{K,j}}_{K,j}$. Carrying out the above procedures recursively, we can eventually have $c_r m_1^{\boldsymbol{\alpha}_1^r} m_2^{\boldsymbol{\alpha}_2^r}...m_K^{\boldsymbol{\alpha}_K^r}$. For all the remaining terms in the polynomial $p(\cdot)$, performing the same steps will lead to completion of our proof. 
\end{proof}

\begin{example}
$p(x^1_1,x^1_2,x^2_1,x^2_2)$ is a polynomial that is permutation invariant within each group, $G_1 = \{x^1_1,x^1_2 \}$ and $G_2 = \{x^2_1,x^2_2 \}$. We let
\begin{align*}
p(x_{1,1},x_{1,2},x_{2,1},x_{2,2}) =& x_{1,1}x_{1,2}^2x_{2,1}x_{2,2}^2 + x_{1,1}^2x_{1,2}x_{2,1}^2x_{2,2} + x_{1,1}x_{1,2}^2x_{2,1}^2x_{2,2} + x_{1,1}^2x_{1,2}x_{2,1}x_{2,2}^2 \\
&+ x_{1,1}^2x_{1,2}^3x_{2,1}^3x_{2,2}^4 + x_{1,1}^3x_{1,2}^2x_{2,1}^4x_{2,2}^3 + x_{1,1}^2x_{1,2}^3x_{2,1}^4x_{2,2}^3 + x_{1,1}^3x_{1,2}^2x_{2,1}^3x_{2,2}^4.
\end{align*}
It is easy to observe that $p(x_{1,1},x_{1,2},x_{2,1},x_{2,2})$ can be rewritten as
\begin{align*}
p(x_{1,1},x_{1,2},x_{2,1},x_{2,2}) =& (x_{1,1}x_{1,2}^2 + x_{1,1}^2x_{1,2})(x_{2,1}x_{2,2}^2 + x_{2,1}^2x_{2,2})\\
& + (x_{1,1}^2x_{1,2}^3 + x_{1,1}^3x_{1,2}^2)(x_{2,1}^3x_{2,2}^4 + x_{2,1}^4x_{2,2}^3) \\
&= m_1^{(2,1)} m_2^{(2,1)} + m_1^{(3,2)} m_2^{(4,3)}.
\end{align*}
\end{example}

\begin{lemma} \label{lemma:mono} For a symmetric polynomial $p(x_1,..., x_n)$ of n variables,  it can be expressed by a polynomial in the power sum symmetric polynomials $p_k(x_1,..., x_n)$ for $1 \leq  k \leq n$ with rational coefficients.
\end{lemma}

\begin{proof}
This lemma is a direct result of the fact that $p_1, p_2, ...,p_n$ are algebraically independent and the ring of symmetric polynomials with rational coefficients can be generated as a $\mathbb{Q}$-algebra, i.e. $\mathbb{Q}[p_1, p_2, ...,p_n]$.~\cite{macdonald1998symmetric,stanley2001enumerative} This lemma can also be easily proved by the combination of the fundamental theorem of symmetric polynomials and newton's identities.
\end{proof}

\noindent\textbf{Proof of Theorem~\ref{theorem6}}
\begin{proof}
By Corollary~\ref{corollary:approx}, for any function $f:X\rightarrow \mathbb{R}$ that are permutation-invariant w.r.t. the variables within each group $G_k, \forall k = 1,...,K$, which is in the form
\[
f(\underbrace{x_{1,1}, x_{1,2}, \cdots, x_{1,N_1}}_{G_1}, \underbrace{x_{2,1}, x_{2,2}, \cdots, x_{2,N_2}}_{G_2}, \cdots, \underbrace{x_{K,1}, x_{K,2}, \cdots, x_{K,N_K}}_{G_K}),
\]
we can always find a polynomial function $p:X\rightarrow \mathbb{R}$ that are also permutation-invariant w.r.t. the variables within each group $G_k$ to approximate $f$ closely in a given small error tolerance $\epsilon$.

Here we first define the function $g_k: \mathbb{R}\rightarrow \mathbb{R}^{N_k}$  in the following form,
\[
g_k(x_{k,n}) = \begin{bmatrix}
x_{k,n}\\ 
x^2_{k,n}\\
x^3_{k,n}\\ 
\vdots \\ 
x^{N_k}_{k,n}
\end{bmatrix}
\]
which thus leads to
\[
\sum_{n=1}^{N_k} g_k(x_{k,n}) = \begin{bmatrix}
\sum_{n=1}^{N_k} x_{k,n}\\ 
\sum_{n=1}^{N_k} x^2_{k,n}\\
\sum_{n=1}^{N_k} x^3_{k,n}\\ 
\vdots \\ 
\sum_{n=1}^{N_k} x^{N_k}_{k,n}
\end{bmatrix} = \begin{bmatrix}
p_1(x_{k,1},\cdots, x_{k,N_k})\\ 
p_2(x_{k,1},\cdots, x_{k,N_k})\\
p_3(x_{k,1},\cdots, x_{k,N_k})\\ 
\vdots \\ 
p_{N_k}(x_{k,1},\cdots, x_{k,N_k})
\end{bmatrix}
\]
Therefore, we generate a sequence of power sums basis by $\sum_{n=1}^{N_k} g_k(x_{k,n})$. 

By Lemma~\ref{lemma:repre}, the polynomial function $p(x_{1,1}, x_{1,2}, \cdots, x_{1,N_1},\cdots, x_{K,1}, x_{K,2}, \cdots, x_{K,N_K})$ can be expressed as $\sum_{i=1} c_i m_1^{\boldsymbol{\lambda}^i_{1}} m_2^{\boldsymbol{\lambda}^i_{2}} ... m_{K}^{\boldsymbol{\lambda}^i_{K}}$. Note that each $m_k^{\boldsymbol{\lambda}^i_{k}}$ is a symmetric polynomial, which thus can be rewritten as a polynomial expression of power sum basis,
\[p_1(x_{k,1},\cdots, x_{k,N_k}), p_2(x_{k,1},\cdots, x_{k,N_k}),\cdots, p_{N_k}(x_{k,1},\cdots, x_{k,N_k}).
\]
which has been generated by $\sum_{n=1}^{N_k} g_k(x_{k,n})$.

Hence the polynomial $p(x_{1,1}, x_{1,2}, \cdots, x_{1,N_1},\cdots, x_{K,1}, x_{K,2}, \cdots, x_{K,N_K})$ is also a function of $\sum_{n=1}^{N_k} g_k(x_{k,n}), \forall k = 1,...,K$, which will be expressed as
\small
\[
p(x_{1,1}, x_{1,2}, \cdots, x_{1,N_1},\cdots, x_{K,1}, x_{K,2}, \cdots, x_{K,N_K}) = h(\sum_{n=1}^{N_1} g_1(x_{1,n}),\sum_{n=1}^{N_2} g_2(x_{2,n}),\cdots, \sum_{n=1}^{N_K} g_K(x_{K,n})).
\]
\normalsize

Thus, the function $f$ could be approximate in any given error tolerance $\epsilon$ by a polynomial $h(\sum_{n=1}^{N_1} g_1(x_{1,n}),\sum_{n=1}^{N_2} g_2(x_{2,n}),\cdots, \sum_{n=1}^{N_K} g_K(x_{K,n}))$, which finishes our proof.

\end{proof}

\end{document}